\documentclass[conference]{IEEEtran}
\usepackage{times}
\usepackage{amsmath}
\usepackage{amsthm}
\usepackage{amsfonts}
\usepackage{algorithm} 
\usepackage{algpseudocode}
\usepackage{booktabs}
\usepackage{multirow}

\theoremstyle{definition}
\newtheorem{definition}{\textbf{Definition}}  
\newtheorem{theorem}[definition]{\textbf{Theorem}}   
\newtheorem{lemma}[definition]{\textbf{Lemma}}
\newtheorem{proposition}[definition]{\textbf{Proposition}}

\newtheorem{assumption}{Assumption}
\DeclareMathOperator{\dist}{dist}

\theoremstyle{remark}
\newtheorem{remark}[definition]{\textbf{Remark}}

\usepackage{xcolor}
\usepackage{bm}
\definecolor{bestred}{RGB}{150,40,40}
\definecolor{secondblue}{RGB}{30,70,140}
\newcommand{\val}[1]{\ensuremath{#1}}
\newcommand{\best}[1]{\textcolor{bestred}{\ensuremath{\bm{#1}}}} 
\newcommand{\second}[1]{\textcolor{secondblue}{\ensuremath{\bm{#1}}}}

\usepackage[numbers]{natbib}
\usepackage{multicol}
\usepackage[bookmarks=true]{hyperref}
\usepackage{graphicx}
\usepackage{subcaption}
\hypersetup{hidelinks}

\begin{document}

\title{IMPACT: An Implicit Active-Set Augmented Lagrangian for Fast Contact-Implicit Trajectory Optimization}

\author{%
\begin{tabular}{@{}c@{}}
Jiayun Li$^{1,2}$ \quad
Dejian Gong$^{1}$ \quad
Georgia Chalvatzaki$^{1,2,3}$\\
$^{1}$PEARL Lab, Dept. of Computer Science, TU Darmstadt, Germany\\[-0.1em]
$^{2}$Hessian.AI, Darmstadt, Germany, 
$^{3}$Robotics Institute Germany (RIG)
\end{tabular}
}


%

\maketitle

\begingroup
\renewcommand{\thefootnote}{}
\footnotetext[0]{Project page is available at \url{https://jonaspflaume.github.io/impact_info/}.}
\endgroup

\begin{abstract}
\textit{Contact-implicit trajectory optimization} (CITO) has attracted growing attention as a unified framework for planning and control in contact-rich robotic tasks. Recent approaches have demonstrated promising results in manipulation and locomotion without requiring a prescribed contact-mode schedule. It is well known that the underlying \textit{mathematical programs with complementarity constraints} (MPCCs) remain numerically ill-conditioned, and systematic, scalable solution strategies for CITO remain an active area of research. More efficient and principled solvers that can handle contact constraints are therefore essential to broaden the applicability of CITO. In this work, we develop an augmented-Lagrangian approach to CITO for solving MPCC-based CITO with stationarity guarantees. The method can be interpreted as \emph{identifying the implicit contact-mode branches on the fly} during the \textit{trajectory optimization} (TO) iterations; we call this approach \textsc{IMPACT} (\textit{IMPlicit contact ACtive-set Trajectory optimization}). We provide an efficient C++ implementation tailored to trajectory-optimization workloads and evaluate it on the open-source CITO and \textit{contact-implicit model predictive control} (CI-MPC) benchmarks. On CITO, \textsc{IMPACT} achieves $2.9\times$--$70\times$ speedups over strong baselines (geometric mean $13.8\times$). On CI-MPC, we show improved control quality for contact-rich trajectories on dexterous manipulation tasks in simulation. Finally, we demonstrate the proposed method on real robotic hardware on a T-shaped object pushing task.

\end{abstract}
\IEEEpeerreviewmaketitle
\section{Introduction}
Optimal motion planning in contact-rich manipulation settings remains one of the major challenges in robotics. The planner must reason not only about continuous dynamics, but also about discrete events in which contacts are established, maintained, and broken to perform tasks such as underactuated manipulation and legged locomotion \cite{bouyarmane2019mcpc,pang2023global}. This hybrid, nonsmooth structure makes it difficult to cast the problem as a well-defined continuous TO problem. A common workaround is to prescribe the contact mode sequence (contact schedule) a priori; conditioned on this discrete choice, the resulting TO problem can often be formulated as a smooth \textit{nonlinear programming} (NLP) and solved efficiently using existing continuous optimization methods \cite{jallet2025proxddp,mastalli2020crocoddyl}. By contrast, when the contact schedule is not fixed, the problem is typically posed as CITO. While CITO is promising for autonomously inferring contact modes in complex tasks, it is generally more challenging to solve due to nonsmooth contact dynamics and the combinatorics of mode changes \cite{yunt2006trajectory,yunt2007combined,yunt2011augmented,posa2014direct,onol2019contact,sleiman2023versatile}.

\begin{figure}
    \centering
    \includegraphics[width=0.95\linewidth]{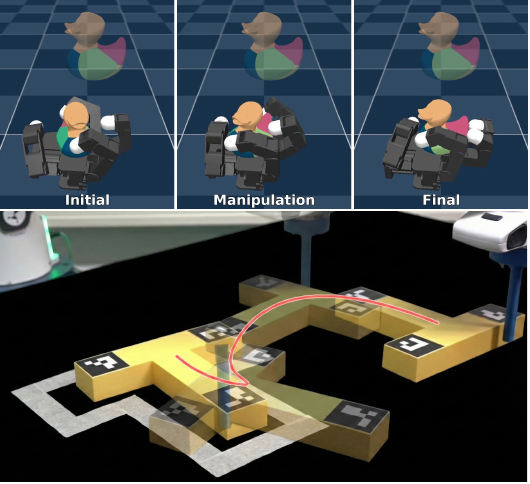}
    \caption{\textsc{IMPACT} demonstrations in simulation and hardware. \textbf{Top:} Allegro Hand reorients a rubber duck in simulation. \textbf{Bottom:} a Panda robot pushes a T-shaped block to a target pose in real-robot experiments; the red curve indicates the trajectory of the block's coordinate origin during the push.}
    \label{fig:IMPACT_demo_allegro_pushT}
    \vspace{-3.2mm}
\end{figure}

CITO commonly relies on contact models with a complementarity structure to capture the on/off nature of contact and frictional interaction, and the exact form depends on the chosen contact model (e.g., rigid vs.\ compliant, Coulomb vs.\ regularized friction) \cite{le2024contact}. This yields a fundamentally nonsmooth optimization problem that is naturally cast as MPCCs \cite{scheel2000mathematical}. At feasible points, complementarity induces degeneracy and violates the regularity assumptions required by standard smooth NLP constraint qualifications (CQs); consequently, standard constraint qualifications such as the \textit{Linear Independence Constraint Qualification} and the \textit{Mangasarian--Fromovitz Constraint Qualification} typically fail, and the associated Lagrange multipliers may become unbounded, rendering the \textit{Karush–Kuhn–Tucker} (KKT) conditions invalid \cite{scheel2000mathematical,nurkanovic2023solving}. Consequently, off-the-shelf NLP solvers based on standard KKT/CQ theory are often numerically brittle for MPCCs: large initial complementarity violations make convergence highly initialization-sensitive and unreliable at TO scale~\cite{li2025surprising}.

Due to the nonsmooth and degenerate nature of contact-implicit formulations, a variety of practical treatments have been explored to improve numerical robustness. Broadly, existing methods modify the original complementarity structure through smoothing or relaxation, or by introducing (exact) penalty terms to better fit standard solvers and their convergence assumptions \cite{kurtz2026inverse, jin2024complementarity, katayarna2022quasistatic}. While these treatments often improve numerical behavior, they can degrade control quality and introduce additional heuristic update parameters that strongly affect efficiency. Moreover, as these modifications are tightened to more closely approximate the original complementarity conditions, degeneracy and ill-conditioning can re-emerge, limiting reliability at TO tasks. 

In this work, we propose \textsc{IMPACT} (\textit{IMP}licit contact \textit{AC}tive-set Trajectory optimization), an \emph{implicit complementarity/contact-branch selection} method for solving MPCC-based CITO, with \emph{stationarity guarantee} for feasible accumulation points under standard assumptions from recent augmented Lagrangian (AuLa) theory~\cite{jia2023augmented, guo2022new, de2023constrained, kanzow2022convergence}. Instead of smoothing or relaxing complementarity constraints, \textsc{IMPACT} retains the original nonsmooth contact constraints and handles them through the AuLa subproblems. Across these subproblems, the approximate Lagrange multipliers and penalty strength guide the selection of the active complementarity branch (contact vs.\ no-contact), enabling an implicit contact-mode discovery without a prescribed homotopy schedule. Fig.~\ref{fig:algorithm_comarison} conceptually contrasts relaxation, penalty, and \textsc{IMPACT} on a toy complementarity example. Moreover, this AuLa treatment prevents the multiplier blow-up typically seen in vanilla AuLa formulations for MPCCs. As a result, \textsc{IMPACT} provides a principled alternative to penalty-only methods. Combined with our proposed \textit{block coordinate descent} (BCD) solver for the AuLa subproblems, \textsc{IMPACT} achieves substantially faster convergence than strong baselines in our experiments while maintaining competitive performance. On standard benchmarks, this translates to an order-of-magnitude speedup. Our contributions are summarized as follows.

\begin{itemize}
  \item \textbf{Algorithm.} We introduce \textsc{IMPACT}, an efficient CITO method based on an AuLa scheme for MPCCs, and establish its stationarity guarantee.
  \item \textbf{Implementation.} We provide C++ implementations of \textsc{IMPACT} and the evaluation baselines, with an interface designed specifically for CITO/CI-MPC workloads.
  \item \textbf{Validation.} We benchmark \textsc{IMPACT} against representative baselines on open-source suites: (i) a CITO benchmark and (ii) a MuJoCo dexterous-manipulation benchmark for Allegro-Hand CI-MPC. We further demonstrate the method on real robotic hardware in a Push-T manipulation task.
\end{itemize}

\section{Related Work}
\label{sec:related}
This section reviews CITO and CI-MPC methods from two perspectives: (a) approaches that \emph{soften} contact logic to obtain smooth gradients for planning/MPC, and (b) approaches that \emph{retain} nonsmooth contact logic and solve more directly.
\subsection{Smooth / Relaxed Contact Models for Planning and MPC}
Many CITO/CI-MPC pipelines replace rigid complementarity with \emph{smooth} or \emph{relaxed} surrogates to enable gradient-based optimization, which can improve numerical behavior but may affect gradient accuracy and closed-loop performance.
A common approach modifies the \emph{forward contact model} (e.g., compliant or differentiable contact), yielding smoother derivatives at the expense of strict rigid-contact fidelity \cite{onol2019contact,pang2023global,suh2022bundled,jin2024complementarity,kurtz2026inverse}.
Another line keeps \emph{forward simulation} rigid but targets \emph{differentiable optimization} by regularizing the sensitivity/KKT systems used to compute gradients through contact, enabling fast CI-MPC and relaxed-complementarity DDP variants \cite{le2024fast,kim2022contact,kim2025contact}. Beyond local smoothing, global convex relaxations have been explored for CITO, emphasizing geometric reasoning and contact sequencing \cite{graesdal2024towards}.

In contrast, our method adopts a \emph{nonsmooth} AuLa viewpoint: we update relaxation via safeguarded multiplier/penalty rules and solve nonsmooth subproblems that explicitly enforce complementarity, aiming to preserve mode transitions and implicitly identify the contact mode.

\subsection{Direct Nonsmooth Treatments of Contact Complementarity}
A second class handles contact complementarity more \emph{directly}, reflecting the inherent nonsmoothness of contact (make/break; stick/slide). Two common strategies are penalty/continuation methods and operator-splitting schemes.

Early CITO work introduces elastic (slack) variables and penalizes complementarity violations so that contacts can be ``discovered'' during optimization \cite{posa2014direct}.  Later pipelines improve reliability via continuation strategies or automatic penalty-update rules \cite{onol2020tuning,li2025surprising}, and related penalty reformulations also appear in \textit{quadratic programs with complementarity constraints} (QPCC) solvers such as LCQPow \cite{hall2025lcqpow}.

Despite their widespread use, penalty/continuation methods can become numerically stiff and brittle under large penalties, and heuristic schedules may slow convergence \cite{li2025surprising}. 
These issues motivate safeguarded AuLa schemes for \emph{general} MPCCs; among penalty-based CITO planners, CRISP is closest to our approach \cite{li2025surprising}.

\emph{Operator splitting / ADMM} exploits CI-MPC structure by alternating trajectory updates with projections onto contact-feasible sets. 
The \textsc{C3} family formulates CI-MPC in a consensus ADMM form with simple, parallelizable subproblems, scaling well to many-contact settings such as planar pushing and multi-contact manipulation \cite{aydinoglu2022real,aydinoglu2024consensus}. 
Extensions improve robustness via residual learning \cite{huang2024adaptive} and sampling guidance \cite{venkatesh2025approximating}, while \textsc{C3+} accelerates planar pushing through more efficient projections \cite{bui2025push}. 
While \textsc{C3+} is reminiscent of our emphasis on fast inner solves, it targets an ADMM projection setting that differs from \textsc{IMPACT}.

The ADMM-based methods above typically operate on linearized contact/dynamics. 
They often produce task-space commands, leading to QPCC-style problems and, in many implementations, an additional task-space tracking layer is required for execution \cite{aydinoglu2022real,bui2025push}. 
A recent preprint by Ménager et al. \cite{menager:hal-05201780}, released after the initial submission of \textsc{IMPACT}, also uses an AuLa and alternating-minimization/BCD scheme for contact-implicit QPCCs in one-step inverse dynamics. 
In contrast, \textsc{IMPACT} targets temporally coupled TO and CI-MPC settings with general MPCC-based CITO formulations. 
We use an MPCC-tailored safeguarded AuLa scheme with complementarity-enforcing inner subproblems. 
This induces an implicit contact-mode selection behavior as slack is driven tight.

\section{Problem Definition}

\begin{figure}
    \centering
    \includegraphics[width=\linewidth]{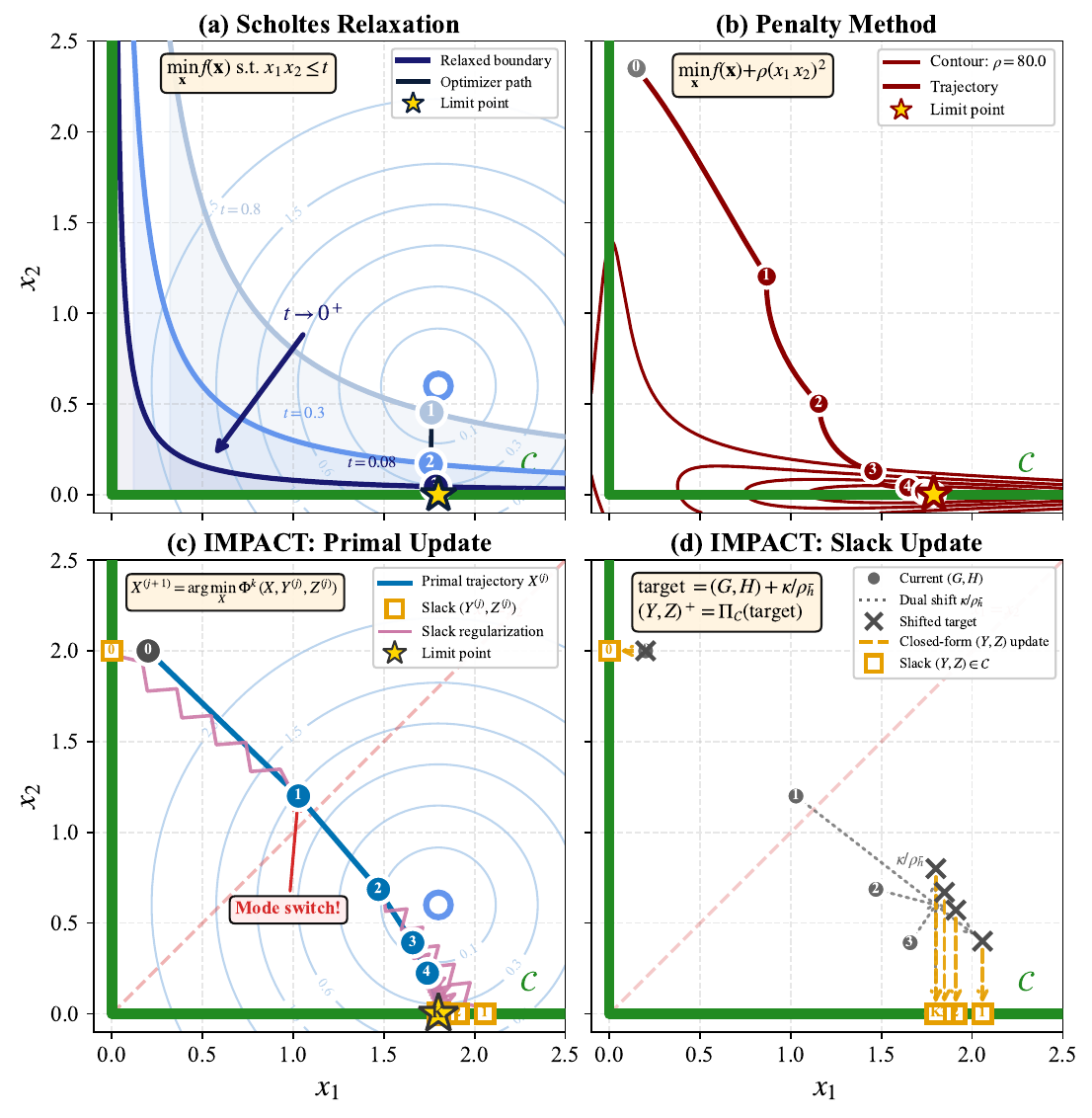}
    \caption{Comparison of complementarity-handling methods on a 2D toy objective. Top row: Scholtes relaxation and squared-penalty method. Bottom row: \textsc{IMPACT}, split into the primal update and the slack update. Objective contours are shown in blue and the complementarity-feasible set is shown in green. In \textsc{IMPACT}, the slack variables act as anchors that regularize the next primal update, while the slack update is obtained from a dual-shifted target projected onto the complementarity set.}
    \label{fig:algorithm_comarison}
    \vspace{-3.2mm}
\end{figure}

In this work, we formulate both CITO and CI-MPC in a unified optimization framework in generalized coordinates, given in problem~(\ref{eq:MPCC}). We use $X = [x_t]_{1:T}$ to denote the trajectory of decision variables (e.g., state, control, and contact forces) over time steps $t=1,\dots,T$. We use $r(x_t)$, $h(x_t, x_{t+1})$, and $g(x_t)$ to denote the residual function, equality constraints, and inequality constraints, respectively. $h(x_t, x_{t+1})$ consists of the standard equality constraints at time steps $t$ and $t+1$, together with the system dynamics that couple consecutive time steps. The $G(x_t)$ and $H(x_t)$ denote the complementarity functions at time step $t$. The complementarity condition $0 \le G(x_t) \perp H(x_t) \ge 0$ encodes a logical exclusivity: component-wise, at least one of $G(x_t)$ or $H(x_t)$ must be zero, so they cannot be simultaneously strictly positive. Algebraically, it is equivalent to $G(x_t)\ge 0$, $H(x_t)\ge 0$, and the element-wise product constraint $G(x_t)\circ H(x_t)=0$. Consequently, the problem is inherently nonconvex and nonsmooth. We focus on nonlinear sum-of-squares (SOS) objectives, which capture many task-space motion-planning costs. It is particularly convenient under the AuLa formulation in Section~\ref{sec:AuLaMPCC}, since both the penalty terms and Lagrangian terms can be absorbed into the nonlinear least-squares residual, yielding an efficient Gauss-Newton scheme.
\begin{subequations}
\label{eq:MPCC}
\begin{align}
\min_{X}\quad & \frac{1}{2}\sum_{t=1}^{T} r(x_t)^{\top} r(x_t) \\
\text{s.t.}\quad
& h(x_t, x_{t+1}) = 0, \quad t \in {1,\dots,T-1}, \\
& g(x_t) \le 0, \quad t \in {1,\dots,T}, \\
& 0 \le G(x_t) \perp H(x_t) \ge 0, \quad t \in {1,\dots,T}.
\end{align}
\end{subequations}

Problem~(\ref{eq:MPCC}) defines a general MPCC, with different discretizations and contact models recovered as special cases. For example, the template accommodates multiple friction models. With a polyhedral approximation of the friction cone, the per-step contact complementarity conditions can be written as a \textit{linear complementarity problem} (LCP). In this case, once the contact Jacobians are fixed at each step (e.g., via a standard linearization or a “frozen Jacobian” approximation), the resulting $G(x_t)$ and $H(x_t)$ are typically affine in the decision variables. If instead friction is enforced using a second-order cone representation, the contact conditions lead to a conic complementarity formulation; this still fits within the same MPCC template, but is generally nonlinear and more challenging to solve. 

In particular, our CI-MPC experiments in \ref{subsec:ci-mpc} adopt the LCP contact model following \cite{jin2024complementarity} and use a time-stepping impulse-velocity formulation in the spirit of Anitescu \cite{anitescu2006optimization}. We directly adopt the KKT system of the primal problem of the Anitescu quadratic program in \cite{jin2024complementarity} (Eq.~(7)) and incorporate it into the \textsc{IMPACT} formulation:
\begin{subequations}
\begin{align}
\mathbf{Q}\mathbf{v}_t - \big(\mathbf{b}(\mathbf{u}_t) + \Tilde{\mathbf{J}}^{\top}_{i}\boldsymbol{\beta}_t\big)/h &= 0, \\
0 \leq \boldsymbol{\beta}_t \perp \Tilde{\mathbf{J}}_{i}\mathbf{v}_t + \phi_i/h &\geq 0 .
\end{align}
\end{subequations}
This system describes the contact dynamics at time step $t$, where the first line gives the discrete-time dynamics, relating the generalized velocity $\mathbf{v}_t$ through the system mass moment $\mathbf{Q}$ to the actuation/external term $\mathbf{b}(\mathbf{u}_t)$ and the contact-force variable $\boldsymbol{\beta}_t$ through the contact Jacobian $\Tilde{\mathbf{J}}_i$ for the $i$th detected contact, including the friction-polyhedral approximation and the step size $h$. The second line imposes complementarity between the contact force
$\boldsymbol{\beta}_t$ and the velocity-level contact residual
$\Tilde{\mathbf{J}}_i\mathbf{v}_t+\phi_i/h$. Componentwise, a contact-force
component can be positive only when the corresponding contact constraint is
active, where $\phi_i$ denotes the normal distance for the $i$th contact. The generalized velocity $\mathbf{v}_t$ is then related to the generalized coordinate using time integration, possibly involving quaternion integration.

\section{A Safeguarded AuLa for MPCCs}
\label{sec:AuLaMPCC}
We next describe the safeguarded AuLa outer loop used in \textsc{IMPACT}. This outer loop updates penalty parameters and Lagrange multipliers while keeping complementarity enforced explicitly, and provides a stationarity-guaranteed wrapper for our BCD inner solver.

For algorithmic convenience, we rewrite Problem~(\ref{eq:MPCC}) in \emph{vertical form} \cite{nurkanovic2023solving} by introducing auxiliary (slack) variables $Y$ and $Z$. This moves nonlinearities from the complementarity set into smooth equality constraints: we enforce $G(x_t)=y_t$ and $H(x_t)=z_t$, and impose complementarity directly on the variables via $0\le y_t\perp z_t\ge 0$. While this increases the dimension and adds equality constraints, it yields a simple complementarity set that is well-suited to our inner solver design. The resulting vertical form is:
\begin{subequations}
\label{eq:vertical_form}
\begin{align}
\min_{X,Y,Z}\quad & \frac{1}{2}\sum_{t=1}^{T} r(x_t)^{\top} r(x_t) \\
\text{s.t.}\quad
& h(x_t, x_{t+1}) = 0, \quad t \in {1,\dots,T-1}, \\
& g(x_t) \le 0, \quad t \in {1,\dots,T}, \\
& G(x_t) = y_t , \quad t \in {1,\dots,T}, \\
& H(x_t) = z_t, \quad t \in {1,\dots,T}, \\
& 0 \le y_t \perp z_t \ge 0, \quad t \in {1,\dots,T}.
\end{align}
\end{subequations}

\begin{remark}
The vertical reformulation is equivalent to the original MPCC: at any feasible point, $(y_t,z_t)=(G(x_t),H(x_t))$, so enforcing $0\le y_t\perp z_t\ge 0$ is equivalent to $0\le G(x_t)\perp H(x_t)\ge 0$. We adopt this split because AuLa handles the resulting smooth equalities $G(x_t)-y_t=0$ and $H(x_t)-z_t=0$ effectively in TO settings \cite{toussaint2014novel}, while our inner solver exploits the resulting variable-wise complementarity structure via closed-form update (Section~\ref{sec:inner_solver}).
\end{remark}

We follow a safeguarded AuLa scheme for MPCCs in which complementarity constraints are not absorbed into the augmented Lagrangian function; instead, they are enforced explicitly in each inner solve. Concretely, we handle the smooth equality and inequality constraints via an AuLa treatment, while keeping the complementarity constraints as hard constraints in the inner subproblem. This separation preserves the MPCC structure and avoids the pathological multiplier growth that can arise when complementarity conditions are penalized directly \cite{guo2022new,jia2023augmented}. The resulting AuLa subproblem, with objective denoted by $\Phi(X,Y,Z)$, is
\begin{subequations}
\label{eq:aula_sub}
\begin{align}
\min_{X,Y,Z}\quad 
&\frac{1}{2}\sum_{t=1}^{T} r_t^{\top} r_t
+ \frac{\rho_{\bar{h}}}{2}\Bigl\lVert \bar{h}_t + \tfrac{\kappa_t}{\rho_{\bar{h}}} \Bigr\rVert^{2}
+ \frac{\rho_g}{2}\Bigl\lVert (g_t + \tfrac{\mu_t}{\rho_g})_{+} \Bigr\rVert^{2} \\
\text{s.t.}\quad 
& 0 \le y_t \perp z_t \ge 0,\quad t \in \{1,\dots,T\}.
\end{align}
\end{subequations}
For brevity, we omit explicit dependence on the variables and stack the equality constraints in~(\ref{eq:vertical_form}) at time step~$t$ as
$\bar h_t(x_{t+1}, x_t,y_t,z_t):=[\,h(x_t, x_{t+1});\ G(x_t)-y_t;\ H(x_t)-z_t\,]$. At the terminal time step $t = T$, $h(x_t, x_{t+1})$ reduces to $h(x_t)$, as no dynamics constraint is imposed beyond the terminal state.
The operator $(\cdot)_{+}$ denotes the element-wise positive part, $(x)_{+}:=\max(x,0)$.
We use $\rho_{\bar h}$ and $\rho_g$ as penalty factors for the equality and inequality constraints, with corresponding multipliers $\kappa_t$ and $\mu_t$.
The resulting subproblem~(\ref{eq:aula_sub}) is a nonlinear least-squares problem with simple complementarity constraints. 

To promote global convergence of the outer loop, we use the safeguarded multiplier and penalty updates in Algorithm~\ref{alg:safeguarded_aula}. Updates are applied component-wise across all time steps, and multipliers are clipped to fixed safeguard bounds to ensure they remain bounded. Lines~\ref{line:define_6}--\ref{line:define_12} implement a standard safeguard for penalty updates: we monitor the decrease of the equality-feasibility residual $\|\bar{h}(w^{k+1})\|_{\infty}$ and an inequality complementarity/KKT residual $\|\zeta^{k+1}\|_{\infty}$ computed after the multiplier update; if the combined violation does not decrease sufficiently, we increase the penalty parameters. This is one possible safeguard design, and other update rules that enforce a comparable sufficient-decrease condition can be used as well.

\begin{algorithm}[t]
\caption{Safeguarded AuLa outer loop for MPCCs}
\label{alg:safeguarded_aula}
\begin{algorithmic}[1]
\Require Initial multipliers $(\kappa^0,\mu^0)$, penalties $(\rho_{\bar{h}}^0,\rho_g^0)$, safeguard bounds $\mu_{\max}>0$ and $\kappa_{\min}<\kappa_{\max}$, penalty increase factor $\gamma>1$, constraints violation factor $\eta\in[0,1]$, inner problem tolerances $\{\varepsilon_k\}\downarrow 0$, outer problem tolerances $\epsilon_w$, $\epsilon_{g}$, $\epsilon_{\bar{h}}$.
\State Initialize $w^{0}$
\State Set $k\gets 0$. Initial variable-change $||\Delta w^{0}||_{\infty} \gets \infty$.
\While{$||\triangle w^k||_{\infty} \geq \epsilon_w$ or $ ||(g(X^{k}))_{+}||_{\infty} \geq \epsilon_{g} $ or $||\bar{h}(w_{k})||_{\infty} \geq \epsilon_{\bar{h}}$ }
  \State \textbf{Safeguard multipliers:}
  \[
    \bar{\kappa}^{k} \gets \Pi_{[\kappa_{\min},\kappa_{\max}]}(\kappa^{k}),
    \qquad
    \bar{\mu}^{k} \gets \Pi_{[0,\mu_{\max}]}(\mu^{k}),
  \]
  where $\Pi$ denotes component-wise clipping. \label{line:clipping_safeguard}

  \State \textbf{Inner solve:} approximately solve the AuLa subproblem ~(\ref{eq:aula_sub}) with $\bar{\mu}^{k}$, $\bar{\kappa}^{k}$ and obtain $w^{k+1} = (X^{k+1},Y^{k+1},Z^{k+1})$ such that the inner stationary residual is below $\varepsilon_k$. \label{line:eps_stationary}

  \State Evaluate $\bar{h}_t$, $g_t$ using $w^{k+1}$, $\triangle w^k \gets w^{k+1} - w^k$
  \State \textbf{Multiplier update:}
  \[\kappa^{k+1} \gets \bar{\kappa}^{k} + \rho_{\bar{h}}^{k}\, \bar{h}_t
    ,
    \qquad
    \mu^{k+1}_t \gets \bigl(\bar{\mu}^{k}_t + \rho_g^{k}\, g_t\bigr)_{+},
  \]
  \State $\zeta^{k+1}_t \gets \min\!\bigl(\mu^{k+1}_t,\, -g_t\bigr)$ \label{line:define_6}
  \State \textbf{Penalty update:}
  \If{$k=0$ \textbf{or} $\max\{\|\zeta^{k+1}\|_{\infty},\,\|\bar{h}(w^{k+1})\|_{\infty}\}\le
  \eta\,\max\{\|\zeta^{k}\|_{\infty},\,\|\bar{h}(w^{k})\|_{\infty}\}$}
    \State $\rho_{\bar{h}}^{k+1}\gets \rho_{\bar{h}}^{k}$;\quad $\rho_g^{k+1}\gets \rho_g^{k}$.
  \Else
    \State $\rho_{\bar{h}}^{k+1}\gets \gamma\,\rho_{\bar{h}}^{k}$;\quad $\rho_g^{k+1}\gets \gamma\,\rho_g^{k}$.
  \EndIf \label{line:define_12}

  \State $k\gets k+1$.
\EndWhile
\end{algorithmic}
\end{algorithm}

The accuracy of each inner solve directly affects the convergence behavior of the AuLa outer loop.
To formalize this connection, we introduce a computable KKT residual for the AuLa
subproblem~(\ref{eq:aula_sub}) and use it as the inner-solve accuracy criterion in
line~\ref{line:eps_stationary} of Algorithm~\ref{alg:safeguarded_aula}.  This residual
verifies: (i) stationarity of the augmented objective, (ii) feasibility of the
complementarity set \(0\le Y\perp Z\ge 0\), and (iii) consistency of the
complementarity multipliers with the limiting normal cone of this set.  The full
definition of the KKT residual is given in Appendix~\ref{app:global_convergence_aula}.

\begin{definition}[$\epsilon$-stationarity]
\label{def:stationary_main}
At outer iteration \(k\), let \(\Phi^k\) denote the AuLa objective
in~(\ref{eq:aula_sub}) with safeguarded multipliers and penalty factors.  We say
\(w^k=(X^k,Y^k,Z^k)\) is \emph{\(\epsilon\)-stationary} for~(\ref{eq:aula_sub}) if
\[
    r_{\mathrm{in}}(w^k)
    \doteq
    \inf_{(u,v)\in\mathcal M_M(w^k)}
    r_{\mathrm{in}}(w^k;u,v)
    \le \epsilon,
\]
where \(r_{\mathrm{in}}\) is computed with respect to \(\Phi^k\), and
\(\mathcal M_M(w^k)\) is the branch-dependent multiplier set defined in
Appendix~\ref{app:global_convergence_aula}.
\end{definition}

The residual \(r_{\mathrm{in}}\) is chosen so that the inner and outer accuracy
criteria are consistent.  It gives a computable certificate for the AuLa
subproblem and is compatible with the gradient-mapping measure controlled by
the inner solver introduced later in Section~\ref{sec:inner_solver}.  Therefore,
the condition \(r_{\mathrm{in}}\le\varepsilon\) is both the accuracy requirement
used by the safeguarded AuLa analysis and a natural stopping measure for the
inner routine.  The formal statement and assumptions are deferred to
Appendix~\ref{app:inner_convergence_rigorous}.

We rephrase the global convergence/stationarity result of \cite[Thm.~3]{guo2022new}
(and the closely related statement in \cite[Cor.~4.4]{jia2023augmented})
in terms of the KKT residual:
\begin{theorem}[Global convergence of safeguarded AuLa under vanishing inner residual]
\label{thm:global_aula_from_inner}
Let $\{w^k\}$ be the sequence generated by Algorithm~\ref{alg:safeguarded_aula}. Assume that the safeguarded multipliers $(\bar\kappa^k,\bar\mu^k)$ remain bounded
(e.g., line \ref{line:clipping_safeguard} of Algorithm~\ref{alg:safeguarded_aula}). Suppose the AuLa subproblem~\eqref{eq:aula_sub}
is solved inexactly in the sense that, for each $k$, there exist $(u^k,v^k)$ such that
\[
r_{\mathrm{in}}(w^k;u^k,v^k)\ \le\ \varepsilon_k,
\qquad \varepsilon_k \downarrow 0 .
\]
Let $w^\star$ be a feasible accumulation point of $\{w^k\}$ for~\eqref{eq:vertical_form}
(equivalently, for~\eqref{eq:MPCC}). If the MPCC regularity condition assumed in \cite[Thm.~3]{guo2022new}
holds at $w^\star$, then $w^\star$ is a first-order stationary point of the MPCC
(in the sense of \cite[Thm.~3]{guo2022new}).
\end{theorem}
Therefore, to ensure the \emph{attainability} of stationarity, we design an inner solver to meet the Definition~\ref{def:stationary_main} for any prescribed tolerance $\varepsilon>0$ within finitely many inner iterations. We present such a solver in the next section.

\section{A BCD Inner Solver for AuLa Subproblems}
\label{sec:inner_solver}

We solve the AuLa inner subproblem~\eqref{eq:aula_sub} with a two-block coordinate descent (BCD) scheme. Each BCD iteration alternates between updating the trajectory variables $X$ and updating the auxiliary complementarity variables $(Y,Z)$. 
The $X$-update is carried out by a Gauss-Newton update with globalization, and the $(Y,Z)$-update admits a closed-form solution under the complementarity constraints.

\paragraph{Block 1: $X$-update (Gauss--Newton with globalization)}
With $(Y,Z)$ fixed, the inner objective reduces to a smooth nonlinear least-squares problem in $X$. We solve it using a damped Gauss--Newton method with an Armijo backtracking line search to ensure sufficient decrease of the inner objective $\Phi^k$.

\paragraph{Block 2: $(Y,Z)$-update (closed-form minimization over a union of cones)}
With $X$ fixed, the $(Y,Z)$-subproblem decouples across time steps and complementary pairs. For $i$th scalar pair at time $t$: $(y_{t,i},z_{t,i})$ subject to $0\le y_{t,i}\perp z_{t,i}\ge 0$, the feasible set is the union of two convex cones,
$\mathcal{C} = \{(y,0)\mid y\ge 0\}\cup\{(0,z)\mid z\ge 0\}$. Thus we evaluate two closed-form candidates and select the one yielding the smaller quadratic penalty
contribution in~\eqref{eq:aula_sub}.

With the safeguarded multipliers at outer iteration $k$ and current penalty factor $\rho_{\bar{h}}$:

\noindent \emph{Case 1} ($z_{t,i}=0$):
\begin{align}
y^{\star}_{t,i} = \max\,\Bigl(0,\ G_i(x_t) + \tfrac{1}{\rho_{\bar{h}}}\,\kappa_{G,t,i}\Bigr),\qquad z^{\star}_{t,i}=0.
\label{eq:slack_case_y}
\end{align}
\emph{Case 2} ($y_{t,i}=0$):
\begin{align}
z^{\star}_{t,i} = \max\,\Bigl(0,\ H_i(x_t) + \tfrac{1}{\rho_{\bar{h}}}\,\kappa_{H,t,i}\Bigr),\qquad y^{\star}_{t,i}=0.
\label{eq:slack_case_z}
\end{align}
We then pick between \eqref{eq:slack_case_y} and \eqref{eq:slack_case_z} by comparing their resulting objective values (equivalently, their quadratic penalty contributions). This step is an \emph{exact} minimization over the complementarity set for fixed $X$ and hence never increases $\Phi^k$. We summarize the BCD approach as Algorithm~\ref{alg:inner_bcd}. A graphical depiction of the algorithm is shown in Fig.~\ref{fig:algorithm_comarison}.

\begin{remark}[Why this differs from the ADMM projection in \cite{bui2025push}]
The ADMM projection in \cite{bui2025push} is primarily feasibility-driven in consensus step: it directly projects the copied variable onto the complementarity-feasible set. 
In contrast, our closed-form update corresponds to a \emph{shifted} mapping induced by the current approximate Lagrange multipliers and penalty weights. Unlike an ADMM consensus update, this step does not introduce copied variables; it is obtained by directly minimizing the AuLa subproblem with respect to the slack variables.
This mapping injects the influence of the current multiplier/penalty into the slack variables, acting like a spring force that regularizes the primal update. 
As illustrated in Fig.~\ref{fig:algorithm_comarison}, \textsc{IMPACT} therefore provides more than a pure feasibility projection and can yield substantially faster progress.
\end{remark}

\begin{theorem}[Attainability of $\varepsilon$-stationarity by the BCD inner solver]
\label{thm:inner_attainability_main}
Fix an outer iteration $k$ and consider the BCD inner iterates $\{w^{(j)}\}=\{(X^{(j)},Y^{(j)},Z^{(j)})\}$
generated by Algorithm~\ref{alg:inner_bcd}. Assume: (i) $\Phi^k$ is continuously differentiable, $\nabla\Phi^k$ is Lipschitz on the bounded set visited by the iterates and \(\Phi^k\) is bounded below on \(\{(X,Y,Z):(Y,Z)\in\mathcal C\}\); (ii) the $X$-update is globalized and achieves sufficient decrease (e.g., damped GN with Armijo line search); and (iii) the $(Y,Z)$-update is an exact minimization over the complementarity set for fixed $X$ (as in \eqref{eq:slack_case_y}--\eqref{eq:slack_case_z}). Then, for any tolerance $\varepsilon_k>0$, there exists a finite inner iteration index $j_k$ and multipliers
$(u^{(j_k)},v^{(j_k)})$ such that
\[
r_{\mathrm{in}}\bigl(w^{(j_k)};u^{(j_k)},v^{(j_k)}\bigr)\ \le\ \varepsilon_k.
\]
In particular, the inner BCD loop can be terminated after finitely many iterations with an output $w^k$ that is
$\varepsilon_k$-stationary in the sense of Definition~\ref{def:stationary_main}.
\end{theorem}

The proof is given in Appendix~\ref{app:inner_convergence_rigorous}: sufficient decrease makes the \(X\)-block residual attainable, and exact \((Y,Z)\)-minimization eliminates the complementarity-block part of the KKT residual.

\textbf{Inner-solver guarantee and practical stopping.}
The inner optimizer is \emph{not} intended to certify global optimality of the nonconvex subproblem.
For the purpose of \emph{outer-loop convergence analysis}, we adopt an \emph{attainability} interface:
for any tolerance $\varepsilon>0$, the inner procedure can produce (in finitely many iterations) an iterate whose
stationarity measure is below $\varepsilon$ under suitable multipliers (Theorem~\ref{thm:inner_attainability_main}).
Our implementation, however, uses an \emph{inexact/budgeted} inner solve:
we do not explicitly test stationarity and instead stop when the inner augmented objective stagnates,
$|\Phi^k(w^{(j)})-\Phi^k(w^{(j+1)})|\le \tau_k$.
Empirically, this criterion consistently yields sufficiently small constraint and complementarity residuals
and stable outer-loop progress on all benchmarks; see Appendix~\ref{app:inner_convergence_rigorous} for definitions and additional diagnostics.

\begin{algorithm}[t]
\caption{Inner Solver: BCD for the AuLa subproblem~\eqref{eq:aula_sub}}
\label{alg:inner_bcd}
\begin{algorithmic}[1]
\Require Outer index $k$, safeguarded multipliers $(\bar\kappa^k,\bar\mu^k)$, penalties $(\rho_{\bar h}^k,\rho_g^k)$,
initial $(X^{(0)},Y^{(0)},Z^{(0)})$, stagnation tolerance $\tau_k$
\State Set $w^{(0)}\gets (X^{(0)},Y^{(0)},Z^{(0)})$ and evaluate $\Phi ^k(w^{(0)})$.
\For{$j=0,1,2,\dots$}
  \State \textbf{$X$-update:} compute a damped GN update on $\Phi^k(\cdot,Y^{(j)},Z^{(j)})$ with backtracking line search; obtain $X^{(j+1)}$.
  \State \textbf{$(Y,Z)$-update:} for each $(t,i)$, compute candidates \eqref{eq:slack_case_y} and \eqref{eq:slack_case_z} and select the one with smaller objective; obtain $(Y^{(j+1)},Z^{(j+1)})$.
  \State Set $w^{(j+1)}\gets (X^{(j+1)},Y^{(j+1)},Z^{(j+1)})$ and evaluate $\Phi^k(w^{(j+1)})$.
  \State \textbf{Practical stopping:} if $|\Phi^k(w^{(j)})-\Phi^k(w^{(j+1)})|\le \tau_k$, 
  \State \hspace{0.5cm} \textbf{return} $w^k\gets w^{(j+1)}$.
\EndFor
\end{algorithmic}
\end{algorithm}
\section{Experiments}
In this section, we evaluate our planner on two open-source benchmarks. For long-horizon planning, we use the CITO benchmark from CRISP \cite{li2025surprising}. For dexterous manipulation under multi-contact dynamics, we use the CI-MPC benchmark from \cite{jin2024complementarity}. All experiments are run on a desktop PC with an AMD Ryzen 9 7950X3D CPU, with no GPU acceleration. Our IMPACT implementation uses no explicit multi-threading (e.g., OpenMP/TBB). It links against BLAS/LAPACK, uses Eigen’s sparse LDLT factorization for linear solves, and updates slack pairs using Eigen’s SIMD-vectorized operations. We use CasADi’s C++ interface for modeling and symbolic differentiation \cite{andersson2019casadi}. CasADi CodeGen is disabled for both CITO and CI-MPC.

\subsection{Long Horizon CITO Benchmark}
The CRISP CITO benchmark comprises six nonlinear, hybrid-dynamics, contact-rich problems for long-horizon planning. The CRISP paper compares standard NLP solvers on four tasks. We exclude the one task that requires a manually specified initial guess, which is outside our scope, and evaluate on three tasks: Push Box, Cart Transport, and Push T.

Rather than comparing against generic NLP solvers that do not explicitly handle complementarity constraints and are often brittle on MPCCs (e.g., SNOPT and PROXNLP \cite{li2025surprising}), we compare against solution strategies that are known to perform well in practice: \textbf{Scholtes relaxation (SR)}, the square \textbf{penalty method (PM)}, and \textbf{CRISP}. To handle the long planning horizons in these benchmarks, we formulate all evaluated methods using multiple shooting. We implement SR and PM using the IPOPT C++ interface \cite{biegler2009large}, with symbolic automatic differentiation provided by the CasADi C++ interface. For \textbf{CRISP}, we directly use the released C++ implementation provided by the authors, which is specialized to the benchmark problems.

For \textbf{SR}, we enforce the nonnegativity constraints $G(x)\ge 0$ and $H(x)\ge 0$ explicitly as inequalities, which IPOPT handles reliably, and relax complementarity via $G(x)H(x)\le t$. We then solve a sequence of relaxed problems, initialized with $t=1.0$ and reducing $t$ by a factor of $10$ at each outer iteration until $t=10^{-10}$. In our experiments, this continuation schedule provides a robust trade-off between numerical stability and runtime. The outer loop terminates early when the complementarity residual falls below a prescribed tolerance. For efficiency, we warm-start IPOPT using the previous iterate’s primal and dual variables.

For \textbf{PM}, we keep $G(x)\ge 0$ and $H(x)\ge 0$ as explicit inequality constraints and penalize complementarity using a squared penalty term with a relatively large penalty weight. Empirically, this configuration is stable and often more efficient than starting from a small penalty and gradually increasing it. While a large penalty can make the inner NLP more ill-conditioned, we find IPOPT remains effective in this regime and can be more efficient than an iterative NLP scheme. 
\begin{figure}
    \centering
    \includegraphics[width=0.75\linewidth]{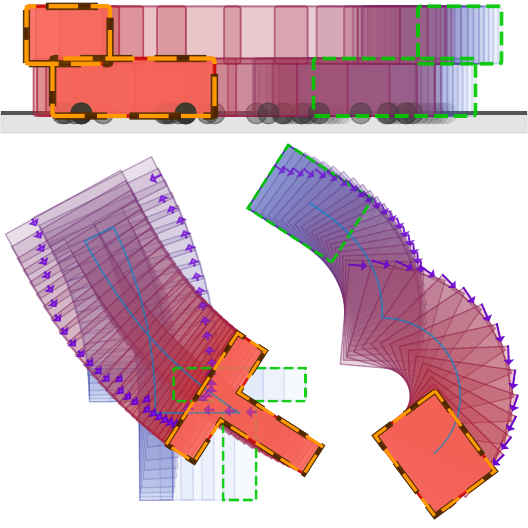}
    \caption{\textsc{IMPACT} planning demos on three CITO tasks. The green dashed box marks the start pose and the orange dashed box marks the goal pose. Purple arrows visualize contact forces, and the solid green line traces the object-origin trajectory. Color indicates time from early (light blue) to late (dark orange).}
    \label{fig:benchmark1}
    \vspace{-3.2mm}
\end{figure}

The benchmark problems are shown in Fig.~\ref{fig:benchmark1}, and we provide detailed descriptions in Appendix~\ref{app:hypper_setup}. We remark that these tasks require forward integration of nonlinear dynamics and include general equality and inequality constraints, so they are not directly cast as LCPs solvable by standard LCP solvers. We report problem sizes in terms of decision variables, complementarity pairs, and constraint counts: \textit{Push Box} has 453 variables, 500 complementarity pairs, and 150 dynamics constraints. \textit{Push T} has 1353 variables, 2150 complementarity pairs, 150 dynamics constraints, 350 equality constraints, and 200 inequality constraints. \textit{Cart Transport} has 2404 variables, 900 complementarity pairs, 1200 dynamics constraints, 300 equality constraints, and 1200 inequality constraints.

\noindent\textbf{Solver Settings and Performance Metrics:}
We use an all-zero initialization for every solver on all tasks, following the standardized CRISP comparison setup. This choice avoids task-specific initial guesses and makes the comparison reproducible across methods. We note, however, that warm-started initialization is common in practical TO/MPC pipelines, and the relative advantage of \textsc{IMPACT} may differ in such settings. For each task, all methods (SR/PM/\textsc{IMPACT} and CRISP) are synchronized to solve the same optimization problem with identical objective-function weights. For \textit{Push Box} and \textit{Push T}, we use a horizon of $T=50$ with $\Delta t=0.05$; for \textit{Cart Transport}, we use $T=300$ with $\Delta t=0.02$. For \textit{Push Box} and \textit{Push T}, we randomly sample 50 goal poses, and for \textit{Cart Transport}, we randomly sample 50 start-goal pairs.

We run each planner until it satisfies a common set of termination criteria. Specifically, we require all constraint violations to be below $10^{-5}$, including (for example) the complementarity violation $\lVert G(x)\circ H(x)\rVert_{\infty} \le 10^{-5}$. A run is declared a failure if (i) it reaches the maximum iteration limit (2000 in our experiments) while still violating either stopping criterion, or (ii) the resulting trajectory fails to reach the goal (e.g., becoming stuck in an intermediate configuration), a failure mode also reported in the CRISP paper. Otherwise, the run is counted as successful. We report success rate, total tracking error, iteration count, and wall-clock runtime in Table~\ref{tab:cito_results}. We report total tracking error as an end-to-end task metric; lower values typically indicate trajectories that reach the goal sooner and maintain better tracking, which often correlates with more efficient contact-mode progress. For \textsc{IMPACT}, we report the number of BCD sweeps, where one sweep corresponds to a full pass over all blocks. For CRISP, we report the total number of QP solves (including those from second-order corrections) as the subproblem-count metric. The MPCC formulations and all solver hyperparameters are provided in Appendix~\ref{app:hypper_setup}.



\begin{table}[t]
\centering
\caption{CITO benchmark results reported as mean $\pm$ 95\% confidence interval (CI). For each task and metric, the best method is highlighted in \textcolor{bestred}{\textbf{red}} and the second best in \textcolor{secondblue}{\textbf{blue}}.}
\label{tab:cito_results}
\setlength{\tabcolsep}{2pt}
\renewcommand{\arraystretch}{1.05}
\begin{tabular}{@{}l|ccccc@{}}
\toprule
Task & Solver & Success (\%) & Track. Err. $\downarrow$ & Iters $\downarrow$ & Time (s) $\downarrow$ \\
\midrule
\multirow{4}{*}{\textit{Box}}
& SR      & \val{100.0} & \best{19.46 \pm 2.30}     & \second{164 \pm 12} & \val{1.27 \pm 0.10} \\
& PM      & \val{98.0} & \val{58.00 \pm 5.11}      & \best{66 \pm 13}    & \second{0.85 \pm 0.17} \\
& CRISP   & \val{100.0} & \val{51.90 \pm 4.08}   & \val{641 \pm 358}   & \val{2.52 \pm 1.34} \\
& IMPACT  & \val{100.0} & \second{26.99 \pm 2.81}      & \val{219 \pm 49}    & \best{0.15 \pm 0.03} \\
\midrule
\multirow{4}{*}{\textit{T}}
& SR      & \val{100.0} & \best{1.24 \pm 0.13}      & \second{202 \pm 11} & \val{5.50 \pm 0.38} \\
& PM      & \val{90.0} & \val{21.99 \pm 2.57}      & \best{117 \pm 14}   & \second{2.96 \pm 0.39} \\
& CRISP   & \val{100.0} & \second{8.03 \pm 0.71}    & \val{1243 \pm 407}  & \val{25.73 \pm 6.92} \\
& IMPACT  & \val{100.0} & \val{8.48 \pm 1.94}       & \val{606 \pm 225}   & \best{1.03 \pm 0.44} \\
\midrule
\multirow{4}{*}{\textit{Cart}}
& SR      & \val{100.0} & \val{433.08 \pm 44.92}    & \val{261 \pm 22}    & \val{5.63 \pm 0.57} \\
& PM      & \val{100.0} & \val{408.65 \pm 42.32}    & \best{87 \pm 14}    & \second{2.00 \pm 0.32} \\
& CRISP   & \val{100.0} & \best{235.10 \pm 21.70}   & \val{501 \pm 381}   & \val{2.72 \pm 2.05} \\
& IMPACT  & \val{100.0} & \second{381.89 \pm 32.35} & \second{101 \pm 15} & \best{0.08 \pm 0.01} \\
\bottomrule
\end{tabular}
\end{table}
Table~\ref{tab:cito_results} shows a quality-speed trade-off across solvers. In terms of \emph{mode/trajectory quality} (total tracking error), SR is best on \textit{Push Box} and \textit{Push T}, while CRISP is best on \textit{Cart Transport}. \textsc{IMPACT} remains competitive: it is second-best on \textit{Push Box} and \textit{Cart Transport}, and close to CRISP on \textit{Push T}. In contrast, PM yields the largest tracking errors and is also less robust, with failures on \textit{Push T} (90\% success) and \textit{Push Box} (98\% success). In terms of \emph{efficiency}, \textsc{IMPACT} provides a dominant runtime advantage. Despite PM having fewer iterations, the penalized subproblems can be ill-conditioned, so wall-clock time does not scale with iteration count. We find that driving the complementarity violation down to the $10^{-5}$ level typically requires multiple continuation iterations in CRISP, leading to increased runtime. SR also remains comparatively expensive, as we find that warm-starting does not substantially reduce the iteration count in practice. Quantitatively, \textsc{IMPACT} is $16.8\times$, $25.0\times$, and $34.0\times$ faster than CRISP on \textit{Box}, \textit{T}, and \textit{Cart}, respectively (geometric mean $24.3\times$). Even relative to the fastest baseline in time (PM), \textsc{IMPACT} achieves $5.7\times$, $2.9\times$, and $25.0\times$ speedups (geometric mean $7.4\times$). Across all baseline comparisons in Table~\ref{tab:cito_results}, this corresponds to a $2.9\times$--$70\times$ speedup range with a geometric mean of $13.8\times$.

\label{subsec:ci-mpc}
\begin{figure*}[t]
    \centering
    \includegraphics[width=\linewidth]{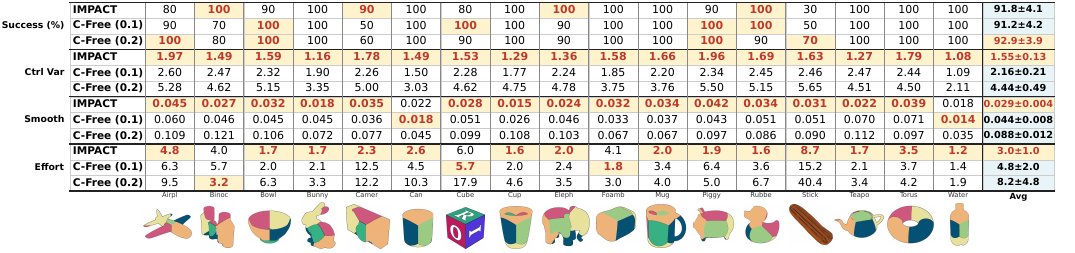}
    \caption{Allegro-Hand CI-MPC benchmark results on 17 objects. We compare \textsc{IMPACT} against \textit{cfree}(0.1) and \textit{cfree}(0.2), which use control bounds $[-0.1,0.1]$ and $[-0.2,0.2]$, respectively. Reported metrics include \textbf{success rate} (\textbf{Success}) and control-quality measures: control variance (\textbf{Ctrl Var}), control smoothness (\textbf{Smooth}), and control effort (\textbf{Effort}). Table columns use these abbreviations. For each object, the best value for each metric is highlighted, and the final column reports the aggregate results with a 95\% confidence region across all 17 objects.}
    \label{fig:CIMPC_benchmark_result}
    \vspace{-3.2mm}
\end{figure*}

\subsection{Real-Robot Push-T Demonstration}
To evaluate \textsc{IMPACT} beyond simulation, we deploy it on a real robotic pushing setup for the Push-T task as a qualitative hardware demonstration (Fig.~\ref{fig:IMPACT_demo_allegro_pushT}). We first solve a full-horizon TO problem to obtain a reference pushing trajectory, and then execute it on the robot by commanding the end-effector to emulate the pusher and track the reference. If the object deviates too much from the nominal trajectory, we replan online. To improve sim-to-real fidelity, we tune the rotational \emph{contact-radius} constraint in the contact model based on observed rollouts. We run $10$ trials with different initial object configurations; \textsc{IMPACT} reaches the target in all trials, achieving a $100\%$ success rate. Additional videos are provided in the supplementary material.

\subsection{Multi-contact High Dimensional CI-MPC Benchmark}
To evaluate the scalability of \textsc{IMPACT} in a realistic CI-MPC setting, we use the Allegro Hand meshable-object reorientation benchmark introduced in \cite{jin2024complementarity}. Following \cite{jin2024complementarity}, we refer to the overall evaluation setup under their framework as \textit{cfree} in the remainder of the paper. As a point of reference, the \textit{cfree} pipeline relies on \textbf{smoothing-based} contact modeling together with the interior-point NLP solver IPOPT, and achieves a $50$-$60 \, \mathrm{Hz}$ CI-MPC on our test PC. However, smoothing-based CI-MPC can exhibit undesired closed-loop behavior: smoothed contact transitions may degrade control smoothness and induce oscillations during execution. In contrast, our goal in this experiment is to demonstrate that \textsc{IMPACT} scales to this benchmark and achieves a practically useful control rate of around $10\,\mathrm{Hz}$ on the same hardware. Importantly, this rate is achieved while directly solving a \emph{multi-contact, nonsmooth, MPCC-based} MPC problem \emph{without} smoothing or relaxation. 
Such problems are widely considered challenging for real-time control due to the combinatorial mode structure and nonsmooth contact dynamics.

The \textit{cfree} benchmark includes 17 objects spanning a wide range of surface curvatures and shapes; see Fig.~\ref{fig:CIMPC_benchmark_result} for the object set. The Allegro-hand setup can involve up to 20 simultaneous detected contact points; the contact complementarity conditions below are instantiated for each detected contact. We adhere closely to the cfree setup and keep all hyperparameters unchanged, including (but not limited to) the simulation time step, friction coefficients, and inertial parameters. The full list of parameters is provided in Appendix~\ref{app:hypper_setup}. 

In \textsc{IMPACT}, we optimize over the generalized velocities ${\mathbf{v}_t}$, contact forces ${\boldsymbol{\beta}_t}$ and control $\mathbf{u}_t$; given the relatively short MPC horizon, we use a single-shooting formulation and eliminate the position states by forward integration. The MPC horizon matches \textit{cfree} and is set to $4$. Following \textit{cfree}, we compute the contact Jacobian only at the MPC initial state and hold it fixed over the prediction horizon.

We keep the same objective structure as \textit{cfree}, but replace its quaternion inner-product term with a log-map rotation residual in $\mathbb{R}^3$, yielding a standard nonlinear least-squares form that better matches our Gauss--Newton pipeline. To keep the relative influence of the rotation terms comparable, we increase its weight accordingly. To improve closed-loop damping, we additionally include a small penalty on object velocity; in contrast, \textit{cfree} primarily introduces damping via penalties on contact forces expressed in the contact frame. We do not perform per-object tuning: as in \textit{cfree}, we use a single shared set of control parameters across all 17 objects. For initialization, we follow the same strategy as \textit{cfree} by adding small random perturbations to the object’s start and goal poses in hand. We also adopt the same success criteria as \textit{cfree}: a squared position-norm error below $0.02$ and a quaternion inner-product error below $0.04$. Each object is evaluated over 10 trials. The total simulation horizon is 1000 time steps with a time step of 0.1 s. One key difference from the default \textit{cfree} setup is the control bound: we use $[-0.1,0.1]$ to promote a locally stable contact mode. For comparison, we also report \textit{cfree} results with control bounds $[-0.1,0.1]$ and $[-0.2,0.2]$. Our \textsc{IMPACT} experiments use MuJoCo's C API via a C++ interface to match our optimizer implementation, whereas we run \textit{cfree} using its released Python implementation. The experimental results are reported in Fig.~\ref{fig:CIMPC_benchmark_result}. Besides the success rate, we report three control-quality metrics. Control variance is defined as $\frac{1}{d}\sum_{i=1}^{d}\mathrm{Var}(\mathbf{b}_i)$, where $d$ is the dimension of the control signal. Control smoothness is computed as $\frac{1}{T-2}\sum_{t=1}^{T-2}\lVert \mathbf{u}_{t+1}-\mathbf{u}_{t}\rVert_2^2$, and control effort as $\sum_{t=1}^{T-1}\lVert \mathbf{u}_t\rVert_2^2$, where $T$ is the trial length.

Fig.~\ref{fig:CIMPC_benchmark_result} summarizes performance across the 17 objects. \textsc{IMPACT} achieves a success rate comparable to \textit{cfree} (avg.\ $91.8\%\pm4.1$ vs.\ $91.2\%\pm4.2$ for \textit{cfree}(0.1) and $92.9\%\pm3.9$ for \textit{cfree}(0.2)), while consistently improving control quality: variance $1.55\pm0.13$ (vs.\ $2.16\pm0.21$ and $4.44\pm0.49$), smoothness $0.029\pm0.004$ (vs.\ $0.044\pm0.008$ and $0.088\pm0.012$), and effort $3.0\pm1.0$ (vs.\ $4.8\pm2.0$ and $8.2\pm4.8$) for \textit{cfree}(0.1/0.2) \cite{jin2024complementarity}. Increasing \textit{cfree}'s control bound from $[-0.1,0.1]$ to $[-0.2,0.2]$ preserves success rate but substantially worsens these control-quality metrics, consistent with more oscillatory closed-loop behavior under aggressive inputs. The \textit{stick} object is a notable outlier where \textsc{IMPACT} performs worse; we attribute this to the heightened sensitivity of long, slender geometries to LCP-based contact modeling, which can induce frequent mode switching. In terms of efficiency, \textsc{IMPACT} completes successful trials in $66.5\pm10.9$ MPC steps, comparable to \textit{cfree}(0.1) $71.4\pm13.5$ and \textit{cfree}(0.2) $57.8\pm14.6$, but runs at a lower control rate of $9.53\pm0.36$ Hz (vs.\ $50.6\pm0.86$ and $54.0\pm1.30$ Hz).

\section{Limitations}

\paragraph{Contact-mode stability under limited compute}
In the Allegro-hand CI-MPC experiments, we observe a transient regime before contact is firmly established in which the optimizer may switch between nearby contact sets. This behavior is most pronounced under tight real-time compute budgets, where only coarse inner solves are feasible and the contact-branch selections can vary across iterations. Despite this, \textsc{IMPACT} scales well to high-dimensional MPCC-based CI-MPC and achieves strong aggregate performance on the benchmark. Incorporating additional stabilization mechanisms (e.g., trust-region \cite{suh2025dexterous} or hysteresis on mode updates) is a promising direction to further reduce contact-set switching in this regime.

\paragraph{Local planning and dependence on contact-encouraging costs}
Under a limited compute budget, \textsc{IMPACT} in the Allegro-hand setup operates as a local planner. While it is effective and yields smooth behavior in many cases, it can be sensitive to local minima and depends on a contact-encouraging cost adopted from \cite{jin2024complementarity}. This limitation is most visible on the stick object, where slender geometry makes contact maintenance highly sensitive. Such heuristic contact terms can increase sensitivity to hyperparameters and contribute to residual oscillations; the smoother Push-T results suggest that global geometric encodings can reduce this reliance. Promising directions include incorporating offline reinforcement learning to provide global guidance and reduce reliance on hand-crafted contact costs, as well as exploring object representations beyond local parameterizations to improve robustness. We leave these directions to future work.

\section{Conclusion} 
\label{sec:conclusion}
We presented \textsc{IMPACT}, an implicit contact active-set augmented Lagrangian method for fast contact-implicit trajectory optimization. \textsc{IMPACT} combines a safeguarded AuLa outer loop for MPCCs with a structured block coordinate descent inner solver, enabling efficient contact-mode discovery without a prescribed continuation schedule while retaining the original nonsmooth complementarity structure.

We implemented \textsc{IMPACT} in C++ and evaluated it on two open-source benchmarks. On the CRISP long-horizon CITO suite, \textsc{IMPACT} achieves substantial runtime gains over strong baselines, with speedups ranging from $2.9\times$ to $70\times$ (geometric mean $13.8\times$). In particular, compared with CRISP, \textsc{IMPACT} is $16.8\times$, $25.0\times$, and $34.0\times$ faster on Push Box, Push T, and Cart Transport, respectively (geometric mean $24.3\times$), while maintaining competitive task-level tracking performance under the same feasibility/stationarity termination criteria.

On the high-dimensional multi-contact CI-MPC benchmark for dexterous manipulation, we demonstrate that the proposed \textsc{IMPACT} architecture scales beyond offline planning to closed-loop control. \textsc{IMPACT} scales to multi-contact dynamics in generalized coordinates and, at comparable success rates to the baseline, yields improved control-quality metrics (lower control variance/smoothness/effort). Finally, we demonstrate the proposed method on real robotic hardware in a Push-T manipulation task.

\section*{Acknowledgments}

This work was supported by the DFG Emmy Noether Programme (CH 2676/1-1), the EU Horizon Europe projects MANiBOT (101120823) and ARISE (101135959), the BMFTR project RIG (16ME1001), and the ERC project SIREN (101163933). We also acknowledge support from the hessian.AI Service Center (BMFTR, 16IS22091), the hessian.AI Innovation Lab (S-DIW04/0013/003), TAM, RAI, Google, and the Alfried Krupp Foundation.

\bibliographystyle{unsrtnat}
\bibliography{references}

\clearpage    
\onecolumn      

\section*{\textbf{Supplementary Material}}
\appendices

In the supplementary material, Appendices~\ref{app:global_convergence_aula} and~\ref{app:inner_convergence_rigorous} clarify the theoretical stationarity interface used by IMPACT. Appendices~\ref{app:hypper_setup} and~\ref{app:additional_exps} provide the hyperparameters used in our experiments and additional results.

Appendix~\ref{app:global_convergence_aula} defines the corrected inner KKT residual $r_{\rm in}$, whose multiplier set is consistent with the limiting normal cone of the complementarity set
\[
    \mathcal C := \{(Y,Z): 0 \le Y \perp Z \ge 0\}.
\]
Appendix~\ref{app:inner_convergence_rigorous} proves that the BCD inner solver can attain this residual to any prescribed tolerance under the stated assumptions.
This provides the inner-solve interface required by the safeguarded AuLa convergence result.

\paragraph{\textbf{Scope}}
Throughout this work, ``convergence/stationarity'' refers to the standard safeguarded AuLa stationarity guarantee for MPCCs: under bounded safeguarded multipliers, vanishing inner residuals
\[
    r_{\rm in}(w^k) \le \epsilon_k,
    \qquad \epsilon_k \downarrow 0,
\]
and the stated MPCC regularity condition, every feasible accumulation point of the outer iterates is M-stationary\footnote{We refer the reader to \cite{nurkanovic2023solving} for the relevant notions of stationarity; these notions are classified by the signs and feasible sets of the multipliers associated with the complementarity constraints. Throughout this paper, we do not otherwise distinguish among them.} for the vertical MPCC formulation.
We do not claim whole-sequence convergence, unconditional feasibility, or existence of accumulation points.
By equivalence of the vertical reformulation, the stationarity statement also applies to the original MPCC after eliminating the slack variables.

\section{Inner stationarity interface for the safeguarded AuLa wrapper}
\label{app:global_convergence_aula}

This appendix makes explicit the inner-stationarity interface used by the safeguarded AuLa wrapper.  The key point is that the complementarity variables are kept as hard constraints in the AuLa
subproblem, and therefore the inner residual must be consistent with the limiting normal cone of the
complementarity set
\[
    \mathcal C
    \doteq
    \{(Y,Z): 0\le Y \perp Z \ge 0\}.
\]
We use the M-stationarity convention of MPCC theory and of the safeguarded AuLa method of Guo and Deng
\cite{guo2022new}.  In particular, the multipliers associated with the complementarity variables are not subject to
a global sign restriction.  Their admissible signs depend on the local complementarity status of each pair
\((Y_i,Z_i)\).

Throughout Appendix~\ref{app:global_convergence_aula} and Appendix~\ref{app:inner_convergence_rigorous}, we
fix an outer iteration \(k\) and write
\[
    \Phi \equiv \Phi^k
\]
for the corresponding inner AuLa objective in~\eqref{eq:aula_sub}.  The inner subproblem is
\[
    \min_{X,Y,Z} \ \Phi(X,Y,Z)
    \qquad
    \text{s.t.}\qquad
    (Y,Z)\in\mathcal C.
\]
The smooth equality and inequality constraints of the original MPCC are already represented in \(\Phi\) through
the safeguarded AuLa terms.  The complementarity set \(\mathcal C\), however, is not penalized or relaxed.

\paragraph{Sign convention}
We use the Lagrangian sign convention
\[
    \mathcal L_{\rm in}(w,u,v)
    \doteq
    \Phi(w)-u^\top Y-v^\top Z,
    \qquad
    w=(X,Y,Z).
\]
Thus the complementarity normal-cone element is represented as
\[
    (-u,-v)\in N_{\mathcal C}(Y,Z),
\]
where \(N_{\mathcal C}\) denotes the limiting normal cone.

\subsection{Stationarity-consistent multipliers and the KKT residual}
\label{app:setup_stationarity}

For a feasible pair \((Y,Z)\in\mathcal C\), define the standard MPCC index sets
\[
    I_{+0}(w)
    \doteq
    \{i:Y_i>0,\ Z_i=0\},
    \qquad
    I_{0+}(w)
    \doteq
    \{i:Y_i=0,\ Z_i>0\},
\]
\[
    I_{00}(w)
    \doteq
    \{i:Y_i=0,\ Z_i=0\}.
\]
We also define
\[
    \mathcal Q_M
    \doteq
    \mathbb R_{++}^{2}
    \ \cup\
    (\mathbb R\times\{0\})
    \ \cup\
    (\{0\}\times\mathbb R).
\]
Equivalently, \((a,b)\in\mathcal Q_M\) if either \(a>0\) and \(b>0\), or \(ab=0\).

\begin{definition}[Stationarity-consistent complementarity multipliers]
\label{def:M_consistent_multipliers}
For a feasible \(w=(X,Y,Z)\) with \((Y,Z)\in\mathcal C\), define
\[
\begin{aligned}
    \mathcal M_M(w)
    \doteq
    \Bigl\{(u,v):\ &
    u_i=0 \quad \forall i\in I_{+0}(w),\\
    &
    v_i=0 \quad \forall i\in I_{0+}(w),\\
    &
    (u_i,v_i)\in \mathcal Q_M \quad \forall i\in I_{00}(w)
    \Bigr\}.
\end{aligned}
\]
No sign restriction is imposed on \(v_i\) for \(i\in I_{+0}(w)\), and no sign restriction is imposed on \(u_i\) for
\(i\in I_{0+}(w)\).
\end{definition}

\begin{lemma}[Normal cone of the complementarity set]
\label{lem:M_normal_cone}
For every feasible \((Y,Z)\in\mathcal C\),
\[
    N_{\mathcal C}(Y,Z)
    =
    \{(-u,-v):(u,v)\in\mathcal M_M(w)\}.
\]
Consequently, the condition
\[
    0\in \nabla_{(Y,Z)}\Phi(w)+N_{\mathcal C}(Y,Z)
\]
is equivalent to the existence of \((u,v)\in\mathcal M_M(w)\) such that
\[
    \nabla_Y\Phi(w)-u=0,
    \qquad
    \nabla_Z\Phi(w)-v=0.
\]
\end{lemma}

\begin{proof}
The complementarity set is a Cartesian product of the two-dimensional cones
\[
    C_i=\{(y_i,z_i):0\le y_i\perp z_i\ge0\}.
\]
For a single pair, the limiting normal cone is
\[
N_{C_i}(y_i,z_i)
=
\begin{cases}
    \{0\}\times\mathbb R, & y_i>0,\ z_i=0,\\[2pt]
    \mathbb R\times\{0\}, & y_i=0,\ z_i>0,\\[2pt]
    \mathbb R_-^2 \cup (\mathbb R\times\{0\})\cup(\{0\}\times\mathbb R), & y_i=0,\ z_i=0.
\end{cases}
\]
Writing each normal element as \((-u_i,-v_i)\) gives exactly the multiplier pattern in
Definition~\ref{def:M_consistent_multipliers}.  Taking the Cartesian product over all complementarity pairs gives
the claim.
\end{proof}

\paragraph{Reference normal-cone stationarity measure}
The stationarity measure required by safeguarded AuLa/MPCC theory is
\begin{equation}
\label{eq:nc_stationarity_rewritten}
    r_{\rm nc}(w)
    \doteq
    \dist_\infty\!\left(
        0,\,
        \nabla\Phi(w)+\{0\}\times N_{\mathcal C}(Y,Z)
    \right),
    \qquad
    w=(X,Y,Z),\quad (Y,Z)\in\mathcal C.
\end{equation}
Here \(\dist_\infty\) denotes distance in the \(\ell_\infty\) norm.  This is the vertical-form version of the inner
stationarity condition used in the safeguarded AuLa method of~\cite{guo2022new}.

\paragraph{\(X\)-block stationarity term}
In the theoretical formulation~\eqref{eq:aula_sub}, the \(X\)-variables are unconstrained after the smooth
constraints have been absorbed into the AuLa objective.  Thus the \(X\)-block residual is
\begin{equation}
\label{eq:x_block_mapping_rewritten}
    G^X(w)
    \doteq
    \nabla_X\Phi(w).
\end{equation}
If an implementation additionally enforces simple convex bounds on \(X\) directly, then \(G^X\) may be replaced
by the usual projected-gradient mapping over that convex set.

\paragraph{Primal complementarity feasibility}
Although the inner subproblem enforces \((Y,Z)\in\mathcal C\) exactly, we keep the primal feasibility term in the
residual to make the KKT certificate explicit:
\begin{equation}
\label{eq:rpri_def_rewritten}
    r_{\rm pri}(w)
    \doteq
    \max\Bigl\{
        \|\min(Y,0)\|_\infty,\,
        \|\min(Z,0)\|_\infty,\,
        \|Y\circ Z\|_\infty
    \Bigr\}.
\end{equation}
For all completed BCD iterates considered below, \(r_{\rm pri}(w)=0\).

\paragraph{KKT residual}
For any feasible \(w=(X,Y,Z)\) and any \((u,v)\in\mathcal M_M(w)\), define
\begin{equation}
\label{eq:rin_corrected_def}
    r_{\rm in}(w;u,v)
    \doteq
    \max\Bigl\{
        \|G^X(w)\|_\infty,\,
        \|\nabla_Y\Phi(w)-u\|_\infty,\,
        \|\nabla_Z\Phi(w)-v\|_\infty,\,
        r_{\rm pri}(w)
    \Bigr\}.
\end{equation}
The best achievable KKT residual is
\begin{equation}
\label{eq:rinM_def}
    r_{\rm in}(w)
    \doteq
    \inf_{(u,v)\in\mathcal M_M(w)}
    r_{\rm in}(w;u,v).
\end{equation}
We call \(w\) \(\varepsilon\)-stationary for the inner AuLa subproblem if
\[
    r_{\rm in}(w)\le\varepsilon.
\]
Equivalently, there exist consistent multipliers \((u,v)\in\mathcal M_M(w)\) such that
\(r_{\rm in}(w;u,v)\le\varepsilon\).

\paragraph{Closed-form evaluation}
The infimum in~\eqref{eq:rinM_def} is separable across complementarity pairs.  Let
\[
    a_i(w)\doteq \nabla_{Y_i}\Phi(w),
    \qquad
    b_i(w)\doteq \nabla_{Z_i}\Phi(w),
\]
and let \((s)_+\doteq \max\{s,0\}\).  For feasible \(w\), define the per-pair stationarity mismatch
\[
d_i(w)
\doteq
\begin{cases}
    |a_i(w)|, & i\in I_{+0}(w),\\[2pt]
    |b_i(w)|, & i\in I_{0+}(w),\\[2pt]
    \displaystyle
    \min\Bigl\{
        \max\{(-a_i(w))_+,(-b_i(w))_+\},\,
        |a_i(w)|,\,
        |b_i(w)|
    \Bigr\}, & i\in I_{00}(w).
\end{cases}
\]
Then
\begin{equation}
\label{eq:rinM_closed_form}
    r_{\rm in}(w)
    =
    \max\Bigl\{
        \|G^X(w)\|_\infty,\,
        r_{\rm pri}(w),\,
        \max_i d_i(w)
    \Bigr\}.
\end{equation}
At a biactive pair \((Y_i,Z_i)=(0,0)\), the third case in \(d_i\) is precisely the
\(\ell_\infty\)-distance from \((a_i,b_i)\) to the multiplier set
\(\mathcal Q_M\).  This is the correction that replaces the incorrect global sign-feasible multiplier restriction.

\subsection{Link to safeguarded AuLa stationarity}
\label{app:link_rin_nc_pg}

\begin{lemma}[KKT residual matches the limiting-normal-cone residual]
\label{lem:rin_vs_normalcone_rewritten}
For every feasible \(w=(X,Y,Z)\) with \((Y,Z)\in\mathcal C\),
\[
    r_{\rm nc}(w)
    =
    r_{\rm in}(w).
\]
\end{lemma}

\begin{proof}
Since \(X\) is unconstrained in the inner AuLa subproblem~\eqref{eq:aula_sub},
\[
    \{0\}\times N_{\mathcal C}(Y,Z)
\]
is the only nonsmooth normal-cone term in~\eqref{eq:nc_stationarity_rewritten}.  By
Lemma~\ref{lem:M_normal_cone},
\[
    (-u,-v)\in N_{\mathcal C}(Y,Z)
    \quad\Longleftrightarrow\quad
    (u,v)\in\mathcal M_M(w).
\]
Therefore,
\[
\begin{aligned}
    r_{\rm nc}(w)
    &=
    \inf_{(u,v)\in\mathcal M_M(w)}
    \max\Bigl\{
        \|\nabla_X\Phi(w)\|_\infty,\,
        \|\nabla_Y\Phi(w)-u\|_\infty,\,
        \|\nabla_Z\Phi(w)-v\|_\infty
    \Bigr\}.
\end{aligned}
\]
Using \(G^X(w)=\nabla_X\Phi(w)\) and \(r_{\rm pri}(w)=0\) for feasible \(w\) gives exactly
\(r_{\rm in}(w)\).
\end{proof}

\begin{proposition}[Safeguarded AuLa interface]
\label{prop:aula_interface_rewritten}
Suppose the inner iterates returned to the outer loop satisfy
\[
    r_{\rm in}(w^k)\le \varepsilon_k,
    \qquad
    \varepsilon_k\downarrow0.
\]
Then the inner stationarity requirement of the safeguarded AuLa method,
\[
    \dist\!\left(
        0,\,
        \nabla\Phi^k(w^k)+\{0\}\times N_{\mathcal C}(Y^k,Z^k)
    \right)
    \le
    O(\varepsilon_k),
\]
holds, up to norm-equivalence constants.  Consequently, under bounded safeguarded multipliers and the MPCC
regularity condition in~\cite[Thm.~3]{guo2022new}, every feasible accumulation point of the outer iterates is
M-stationary for the vertical MPCC formulation.
\end{proposition}

\begin{proof}
The first claim follows immediately from Lemma~\ref{lem:rin_vs_normalcone_rewritten} and norm equivalence between
\(\ell_\infty\) and Euclidean distance.  The second claim is exactly the safeguarded AuLa convergence implication
for MPCCs with the complementarity constraints kept as lower-level constraints; see~\cite[Thm.~3]{guo2022new}.
\end{proof}

\paragraph{Takeaway}
The residual \(r_{\rm in}\) is the sole inner-stationarity certificate used by the outer convergence statement.
It is a computable KKT-style residual, but its multiplier feasibility is normal-cone consistent at the
complementarity set. 

\section{\texorpdfstring{$\varepsilon$}--stationary attainability for the BCD inner solver}
\label{app:inner_convergence_rigorous}

This appendix proves that the BCD inner solver can attain the KKT residual \(r_{\rm in}\) to any
prescribed tolerance in finitely many inner iterations.  This is the attainability property needed by the safeguarded
AuLa outer loop.  The result does not assert whole-sequence convergence of the nonconvex inner solver, nor does it
assert global optimality of the AuLa subproblem.

Throughout this section, \(k\) is fixed and \(\Phi\equiv\Phi^k\).  We write
\[
    w^{(j)}=(X^{(j)},Y^{(j)},Z^{(j)})
\]
for completed BCD iterates after the \((Y,Z)\)-update.  If the initialized value
\((Y^{(0)},Z^{(0)})\) is not the exact minimizer for \(X^{(0)}\), all statements below apply after the first completed
\((Y,Z)\)-update, with a harmless relabeling of the index.

\subsection{Assumptions}
\label{app:assumptions_rewritten}

\begin{assumption}[Smoothness and boundedness]
\label{assump:pg_smooth_bounded_rewritten}
Along the inner iterates \(\{w^{(j)}\}\) generated at the fixed outer iteration \(k\):
\begin{enumerate}
    \item \(\Phi\) is continuously differentiable on a neighborhood of the iterates, and \(\nabla\Phi\) is Lipschitz
    on that neighborhood.
    \item The sequence \(\{w^{(j)}\}\) is bounded, and
    \[
        \inf_{w:\,(Y,Z)\in\mathcal C}\Phi(w)>-\infty.
    \]
\end{enumerate}
\end{assumption}

\begin{assumption}[\(X\)-step sufficient decrease]
\label{assump:pg_x_decrease_rewritten}
There exists a constant \(c_X>0\) such that every completed BCD sweep satisfies
\begin{equation}
\label{eq:x_decrease_vs_pg_rewritten}
    \Phi(w^{(j+1)})
    \le
    \Phi(w^{(j)})
    -
    c_X\,
    \bigl\|G^X(w^{(j)})\bigr\|_\infty^2.
\end{equation}
This assumption covers both one-step and multi-step \(X\)-updates. For example, it holds if the completed
\(X\)-update is globalized by Armijo-type line search so that the accumulated decrease is comparable to the
\(X\)-block residual, and the subsequent \((Y,Z)\)-update does not increase \(\Phi\).
\end{assumption}

\begin{assumption}[Exact \((Y,Z)\)-block minimization over the complementarity set]
\label{assump:pg_yz_exact_rewritten}
For every completed BCD iterate \(w^{(j)}\),
\[
    (Y^{(j)},Z^{(j)})
    \in
    \arg\min_{(Y,Z)\in\mathcal C}
    \Phi(X^{(j)},Y,Z).
\]
\end{assumption}

\subsection{Descent and \texorpdfstring{$X$}--block attainability}
\label{app:descent_attainability_rewritten}

\begin{lemma}[Monotone decrease and square-summable \(X\)-block residual]
\label{lem:pg_descent_sum_rewritten}
Under Assumptions~\ref{assump:pg_smooth_bounded_rewritten} and~\ref{assump:pg_x_decrease_rewritten},
the inner objective values are monotone non-increasing and
\[
    \sum_{j=0}^{\infty}
    \bigl\|G^X(w^{(j)})\bigr\|_\infty^2
    <
    \infty.
\]
\end{lemma}

\begin{proof}
Summing~\eqref{eq:x_decrease_vs_pg_rewritten} from \(j=0\) to \(N\) gives
\[
    c_X
    \sum_{j=0}^{N}
    \bigl\|G^X(w^{(j)})\bigr\|_\infty^2
    \le
    \Phi(w^{(0)})-\Phi(w^{(N+1)}).
\]
By Assumption~\ref{assump:pg_smooth_bounded_rewritten}, the right-hand side is bounded above by
\[
    \Phi(w^{(0)})-\inf_{w:\,(Y,Z)\in\mathcal C}\Phi(w).
\]
Letting \(N\to\infty\) proves square summability.  Monotonicity follows directly from
\eqref{eq:x_decrease_vs_pg_rewritten}.
\end{proof}

\begin{theorem}[Attainability of the \(X\)-block residual]
\label{thm:pg_attainability_rewritten}
Under Assumptions~\ref{assump:pg_smooth_bounded_rewritten} and~\ref{assump:pg_x_decrease_rewritten}, for any
\(\varepsilon>0\), there exists a finite inner iteration index \(j_\varepsilon\) such that
\[
    \bigl\|G^X(w^{(j_\varepsilon)})\bigr\|_\infty
    \le
    \varepsilon.
\]
Moreover, for any \(N\ge0\),
\[
    \min_{0\le j\le N}
    \bigl\|G^X(w^{(j)})\bigr\|_\infty
    \le
    \sqrt{
        \frac{
            \Phi(w^{(0)})-\inf_{w:\,(Y,Z)\in\mathcal C}\Phi(w)
        }{
            c_X(N+1)
        }
    }.
\]
\end{theorem}

\begin{proof}
The quantitative bound follows from the summability estimate in
Lemma~\ref{lem:pg_descent_sum_rewritten}.  The finite-attainability statement follows because the right-hand side
tends to zero as \(N\to\infty\).
\end{proof}

\subsection{From \texorpdfstring{$X$}--block attainability to KKT inner stationarity}
\label{app:full_stationarity_rewritten}

\begin{lemma}[Exact \((Y,Z)\)-minimization eliminates the \((Y,Z)\)-part of the KKT residual]
\label{lem:yz_exact_implies_rinM}
\label{lem:yz_exact_implies_rin_diamond}
Under Assumptions~\ref{assump:pg_smooth_bounded_rewritten} and~\ref{assump:pg_yz_exact_rewritten}, every completed
BCD iterate \(w^{(j)}=(X^{(j)},Y^{(j)},Z^{(j)})\) satisfies
\begin{equation}
\label{eq:rinM_equals_xblock}
    r_{\rm in}(w^{(j)})
    =
    \bigl\|G^X(w^{(j)})\bigr\|_\infty.
\end{equation}
\end{lemma}

\begin{proof}
By Assumption~\ref{assump:pg_yz_exact_rewritten}, \((Y^{(j)},Z^{(j)})\) is a global, hence local, minimizer of the
smooth function
\[
    (Y,Z)\mapsto \Phi(X^{(j)},Y,Z)
\]
over the closed set \(\mathcal C\).  The Fermat rule for closed sets gives
\[
    0
    \in
    \nabla_{(Y,Z)}\Phi(w^{(j)})
    +
    N_{\mathcal C}(Y^{(j)},Z^{(j)}).
\]
By Lemma~\ref{lem:M_normal_cone}, there exist multipliers
\[
    (u^{(j)},v^{(j)})\in\mathcal M_M(w^{(j)})
\]
such that
\[
    \nabla_Y\Phi(w^{(j)})-u^{(j)}=0,
    \qquad
    \nabla_Z\Phi(w^{(j)})-v^{(j)}=0.
\]
The completed \((Y,Z)\)-update is feasible for \(\mathcal C\), so
\[
    r_{\rm pri}(w^{(j)})=0.
\]
Therefore the only nonzero term in the corrected residual~\eqref{eq:rin_corrected_def} is the \(X\)-block term, and
\[
    r_{\rm in}(w^{(j)})
    \le
    \bigl\|G^X(w^{(j)})\bigr\|_\infty.
\]
The reverse inequality follows directly from the definition of \(r_{\rm in}\), since the \(X\)-block term appears
inside the maximum in~\eqref{eq:rin_corrected_def}.  Hence equality holds.
\end{proof}

\begin{theorem}[Attainability of MPCC inner \(\varepsilon\)-stationarity by the BCD inner solver]
\label{thm:inner_attainability_m}
\label{thm:inner_attainability_diamond}
\label{thm:inner_attainability_corrected}
Fix outer iteration \(k\) and consider the completed BCD inner iterates \(\{w^{(j)}\}\) generated by
Algorithm~\ref{alg:inner_bcd}.  Suppose Assumptions~\ref{assump:pg_smooth_bounded_rewritten}--
\ref{assump:pg_yz_exact_rewritten} hold.  Then, for any tolerance \(\varepsilon>0\), there exists a finite inner
iteration index \(j_\varepsilon\) such that
\[
    r_{\rm in}(w^{(j_\varepsilon)})
    \le
    \varepsilon.
\]
Moreover, for any \(N\ge0\),
\[
    \min_{0\le j\le N}
    r_{\rm in}(w^{(j)})
    \le
    \sqrt{
        \frac{
            \Phi(w^{(0)})-\inf_{w:\,(Y,Z)\in\mathcal C}\Phi(w)
        }{
            c_X(N+1)
        }
    }.
\]
Equivalently, there exist multipliers
\[
    (u^{(j_\varepsilon)},v^{(j_\varepsilon)})
    \in
    \mathcal M_M(w^{(j_\varepsilon)})
\]
such that
\[
    r_{\rm in}
    \bigl(
        w^{(j_\varepsilon)};
        u^{(j_\varepsilon)},v^{(j_\varepsilon)}
    \bigr)
    \le
    \varepsilon.
\]
\end{theorem}

\begin{proof}
Combine Theorem~\ref{thm:pg_attainability_rewritten} with
Lemma~\ref{lem:yz_exact_implies_rinM}.  The multiplier statement follows from the definition of
\(r_{\rm in}\).
\end{proof}

\noindent\textbf{Practical stopping by stagnation.}
The implementation uses a budgeted inner solve and may stop by objective stagnation rather than explicitly
evaluating \(r_{\rm in}\).  The theory above is therefore an attainability statement.  Nevertheless, the sufficient
decrease condition gives the following useful calibration.  If
\[
    \Delta_j
    \doteq
    \Phi(w^{(j)})-\Phi(w^{(j+1)})
    \le
    \tau_k,
\]
then~\eqref{eq:x_decrease_vs_pg_rewritten} implies
\[
    \bigl\|G^X(w^{(j)})\bigr\|_\infty
    \le
    \sqrt{\tau_k/c_X}.
\]
By Lemma~\ref{lem:yz_exact_implies_rinM},
\[
    r_{\rm in}(w^{(j)})
    \le
    \sqrt{\tau_k/c_X}.
\]
Thus objective stagnation is a practical proxy for the KKT residual.  The safeguarded AuLa convergence
statement itself, however, uses the theoretical residual \(r_{\rm in}\) and the attainability result in
Theorem~\ref{thm:inner_attainability_m}.

\subsection{Empirical Results for the Stagnation-Based Stopping Rule}
\label{app_subsec: empirical_active_set}

\paragraph{Empirical validation of the stagnation-based stopping rule}
\label{param:emp_stagnation_stopping}
We empirically examine how the stagnation test relates to the corrected inner stationarity measure on the Push-T manipulation task.
Across $30$ random initializations (15{,}139 recorded inner-iteration pairs), we compute the association between
\[
    \|G^X(w^{(j)})\|_\infty
\]
and the one-step objective-decrease proxy
\[
    \sqrt{\Delta\Phi(w^{(j)})},
    \qquad
    \Delta\Phi(w^{(j)}) \doteq \Phi(w^{(j)})-\Phi(w^{(j+1)}).
\]
By Lemma~\ref{lem:yz_exact_implies_rinM}, for completed BCD iterates this quantity coincides with the KKT residual:
\[
    r_{\rm in}(w^{(j)})=\|G^X(w^{(j)})\|_\infty.
\]

\begin{itemize}
    \item \textbf{Strong correlation.}
    In log--log scale, the correlation coefficient is $\rho = 0.983$, indicating that the stagnation signal is a
    reliable empirical proxy for $\|G^X(w^{(j)})\|_\infty$ and hence for $r_{\rm in}(w^{(j)})$.

    \item \textbf{Near-linear scaling.}
    A power-law fit yields an exponent of $0.945 \approx 1$, consistent with the theoretical
    $\mathcal{O}(\sqrt{\tau_k})$ dependence implied by the sufficient-decrease relation.

    \item \textbf{Empirical constant.}
    Empirically, we observe
    \[
        \|G^X(w^{(j)})\|_\infty
        \le c \sqrt{\Delta\Phi(w^{(j)})}
    \]
    with $c \approx 1.9$
    (equivalently, $c_X \approx 0.28$ in~\eqref{eq:x_decrease_vs_pg_rewritten}).
\end{itemize}

\noindent Consequently, on Push-T, terminating the inner BCD loop when
\[
    \Phi(w^{(j)})-\Phi(w^{(j+1)}) \le \tau_k
\]
\emph{typically} yields
\begin{equation}
\label{eq:empirical_stagnation_bound}
    r_{\rm in}(w^{(j)})
    =
    \|G^X(w^{(j)})\|_\infty
    \ \le\ 1.9\,\sqrt{\tau_k}.
\end{equation}
We emphasize that~\eqref{eq:empirical_stagnation_bound} is an empirical calibration on this problem; the theoretical
analysis only requires the existence of a sufficient-decrease constant.

\begin{figure}[htpb]
    \centering
    \includegraphics[width=0.66\linewidth]{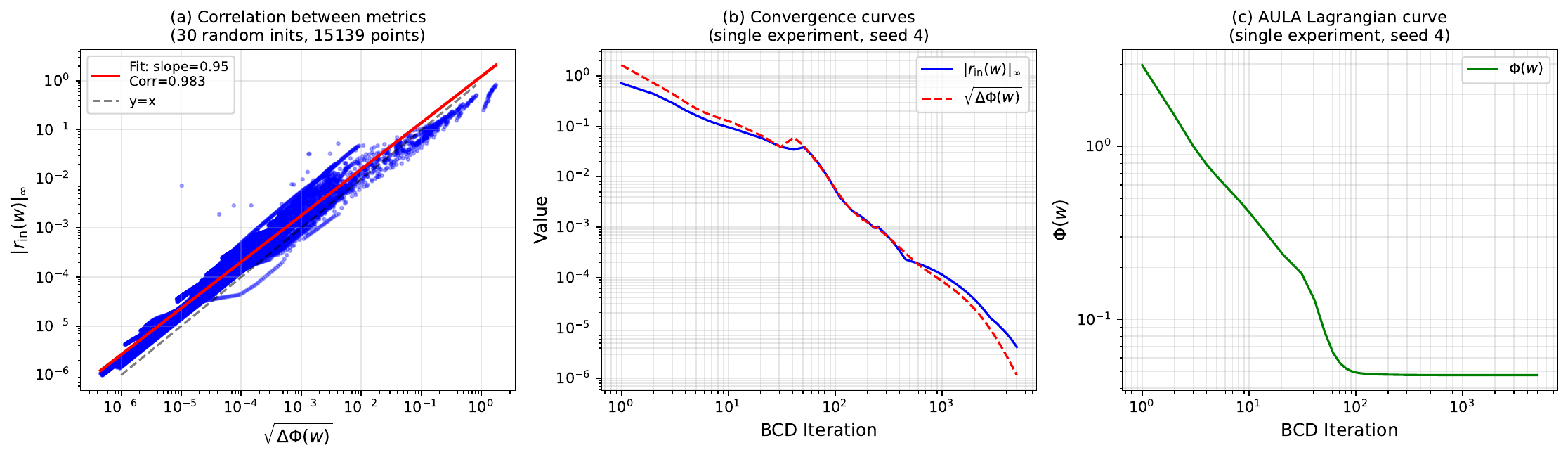}
    \caption{Empirical validation of the stagnation-based stopping criterion on the Push-T task.
    (a) Log-log scatter plot showing the correlation ($\rho = 0.98$) between the $X$-block residual
    $\|G^X(w)\|_\infty$ (equivalently, the corrected inner residual $r_{\rm in}(w)$ on completed BCD iterates)
    and the stagnation measure $\sqrt{\Delta\Phi(w)}$ across 30 random initializations (15,139 data points).
    (b) Convergence of both metrics during BCD iterations for a representative experiment.
    (c) Corresponding value of the augmented Lagrangian $\Phi(w)$.}
    \label{fig:practical_correlation}
\end{figure}

\section{Hyperparameters and Experiments setup}
\label{app:hypper_setup}

\subsection{CITO Benchmark Problems}
\noindent\textbf{Push Box} Push Box is a planar pushing task where a rectangular box moves on a tabletop under friction. The solver must plan a long-horizon trajectory to reach target poses while selecting appropriate contact modes (which face is pushed) and corresponding contact forces, resulting in nontrivial mode switching.

\noindent\textbf{Cart Transport} Cart Transport involves indirect manipulation via friction: a payload rests on a cart and may stick or slip as the cart accelerates. The goal is to drive the system to a target while keeping the payload on the cart, often requiring non-monotone motions that exploit stick--slip transitions to reposition the payload.

\noindent\textbf{Push T} Push T generalizes planar pushing to a non-convex T-shaped object, increasing the number of feasible contact locations and modes. This leads to a larger complementarity system and a more complex mode landscape, making Push T a harder benchmark.

We follow the optimization problem formulation in CRISP~\cite{li2025surprising} (in its Appendix~B). Specifically, CRISP enforces the system dynamics as equality constraints, models unilateral contacts via complementarity constraints, and (for tasks such as Push-T) uses mutual-exclusion complementarity constraints to ensure that at most one control mode is active at each time step. For the Push-T and Cart-Transport tasks, we include the control-feasibility and boundary constraints within the MPCC formulation. In general, the objective terms and constraints are identical to those in CRISP.

We summarize the benchmark settings in Table~\ref{tab:cito_hyperparameters}. The solver parameters are reported for \textsc{IMPACT}; for all other baselines, we set the stopping criteria to match this convergence setup.
\begin{table*}[ht]
\centering
\caption{Hyperparameters for the three benchmark tasks.}
\label{tab:cito_hyperparameters}
\begin{tabular}{l|c|c|c}
\toprule
\textbf{Parameter} & \textbf{Push Box} & \textbf{Push T} & \textbf{Cart Transporter} \\
\midrule
\multicolumn{4}{c}{\textit{Physical Parameters}} \\
\midrule
Mass $m$ (kg) & 0.1 & 0.1 & $m_1=0.1$, $m_2=0.2$ \\
Gravity $g$ (m/s$^2$) & 9.81 & 9.8 & 9.81 \\
Friction coefficient $\mu$ & 0.5 & 0.4 & 0.2 \\
Characteristic length (m) & $a=0.3$, $b=0.4$ & $l=0.05$ & $l=1.0$ \\
Time step $\Delta t$ (s) & 0.05 & 0.05 & 0.02 \\
\midrule
\multicolumn{4}{c}{\textit{Problem Dimensions}} \\
\midrule
State dimension $n_x$ & 3 & 3 & 4 \\
Control dimension $n_u$ & 6 & 24 & 4 \\
Complementarity pairs $n_c$ & 10 & 43 & 3 \\
Equality constraints $n_e$ & 0 & 7 & 1 \\
Inequality constraints $n_i$ & 0 & 4 & 4 \\
\midrule
\multicolumn{4}{c}{\textit{Optimization Parameters}} \\
\midrule
Horizon & 50 & 50 & 300 \\
Stage control cost weight  & 0.001 & 0.01 & $10^{-6}$ \\
Final cost weight  & 100.0 & 100.0 & 5000.0 \\
Multiplier Safeguard & $10^6$ & $10^6$ & $10^6$ \\
$\rho$ scaling factor & 1.1 & 1.1 & 1.5 \\
\midrule
\multicolumn{4}{c}{\textit{Convergence Criteria}} \\
\midrule
Max outer iterations & 1000 & 1000 & 1000 \\
Outer tolerance $\epsilon_h$ & $10^{-5}$ & $10^{-5}$ & $10^{-5}$ \\
Outer tolerance $\epsilon_{\text{comp}}$ & $10^{-5}$ & $10^{-5}$ & $10^{-5}$ \\
Max inner iterations & 50 & 50 & 10 \\
Inner tolerance & $10^{-3}$ & $10^{-3}$ & $10^{-3}$ \\
\midrule
\multicolumn{4}{c}{\textit{Newton Solver Parameters}} \\
\midrule
Max Newton iterations & 50 & 200 & 100 \\
Newton tolerance & $10^{-6}$ & $10^{-6}$ & $10^{-6}$ \\
Regularization & $2\times10^{-5}$ & $5\times10^{-5}$ & $10^{-5}$ \\
\bottomrule
\end{tabular}
\end{table*}

\subsection{Allegro Benchmark}
We summarize the hyperparameters used in the Allegro benchmark tests in Table~\ref{tab:allegro_hyperparameters}.

\begin{table*}[t]
\centering
\caption{Hyperparameters comparison for Allegro Hand in-hand manipulation: IMPACT vs. C-Free}
\label{tab:allegro_hyperparameters}
\begin{tabular}{l|c|c|l}
\toprule
\textbf{Parameter} & \textbf{IMPACT} & \textbf{C-Free} & \textbf{Notes} \\
\midrule
\multicolumn{4}{c}{\textit{System Dimensions}} \\
\midrule
State dimension $n_q$ & 23 & 23 & Same (obj: 7, robot: 16) \\
Velocity dimension $n_v$ & 22 & 22 & Same (obj: 6, robot: 16) \\
Command dimension $n_u$ & 16 & 16 & Same (joint position cmds) \\
Max contacts $n_{\text{con}}$ & 20 & 20 & Same \\
\midrule
\multicolumn{4}{c}{\textit{Simulation Parameters (Mujoco)}} \\
\midrule
Time step $h$ (s) & 0.1 & 0.1 & Same \\
Frame skip & 50 & 50 & Same \\
Friction coefficient $\mu$ & 0.5 & 0.5 & Same \\
\midrule
\multicolumn{4}{c}{\textit{MPC Parameters}} \\
\midrule
Horizon $N$ & 4 & 4 & Same \\
Control bound & $\pm 0.1$ & $\pm 0.1 (0.2)$ &  \\
\midrule
\multicolumn{4}{c}{\textit{Cost Function Weights (Path Cost)}} \\
\midrule
Position cost weight & 0.0 & 0.0 & Same \\
Quaternion cost weight & 0.0 & 0.0 & Same \\
Contact cost weight & 1.0 & 1.0 & Encourage contact \\
Grasp closure weight & 0.0 & 0.0 & Same \\
Control cost weight & 0.1 & 0.1 & Same \\
Velocity penalty & 0.1 & -- & Object Damping \\
\midrule
\multicolumn{4}{c}{\textit{Cost Function Weights (Final Cost)}} \\
\midrule
Final position weight & 1000.0 & 1000.0 & Same \\
Final quaternion weight & 90.0 & 50.0 & scale match \\
\midrule
\multicolumn{4}{c}{\textit{Solver-Specific Parameters}} \\
\midrule
\multicolumn{4}{l}{\textbf{IMPACT Parameters}} \\
$\rho_{\max}$ & $10^3$ & -- & \\
$\rho$ scale factor & 5.0 & -- & \\
Max outer iterations & 10 & -- & \\
Outer tol ($\epsilon_h$) & $10^{-3}$ & -- & \\
Outer tol ($\epsilon_{\text{comp}}$) & $10^{-3}$ & -- & \\
Max inner iterations & 5 & -- & \\
Newton max iterations & 30 & -- & \\
Newton step tolerance & $10^{-5}$ & -- & \\
Newton obj tolerance & $10^{-6}$ & -- & \\
\midrule
\multicolumn{4}{l}{\textbf{C-Free (IPOPT-based NLP)}} \\
IPOPT max iterations & -- & 50 & \\
Complementarity relaxation & -- & $>0$ & Relaxed LCP \\
\midrule
\multicolumn{4}{c}{\textit{Success Criteria}} \\
\midrule
Position tolerance (m) & 0.02 & 0.02 & Same \\
Quaternion tolerance & 0.04 & 0.04 & Same \\
Consecutive success steps & 20 & 20 & Same \\
Max rollout steps & 1000 & 1000 & Same \\
\bottomrule
\end{tabular}
\end{table*}

\section{Additional Experiment Demonstration}
\label{app:additional_exps}
We provide additional experimental demonstrations here to build intuition for the main-text results.
\subsection{IMPACT CITO Extra Demos}
We include additional qualitative demos of IMPACT CITO on three tasks. Figures~\ref{fig:box-2x2} and~\ref{fig:pusht-2x2} report overlaid rollouts for the Pushing Box and PushT settings across multiple start--goal pairs in the $(x,y,\theta)$ state space. Figure~\ref{fig:cart-1x4} shows rollout sequences for the Cart Transporter task in $(x,y,dx,dy)$, with zero start and goal velocities. Visualization conventions (start/goal poses, force direction, and time progression) follow the figure captions. These results are based on the same data as the supplementary video.

\begin{figure}[t]
  \centering

  \begin{subfigure}[t]{0.25\linewidth}
    \centering
    \includegraphics[width=\linewidth]{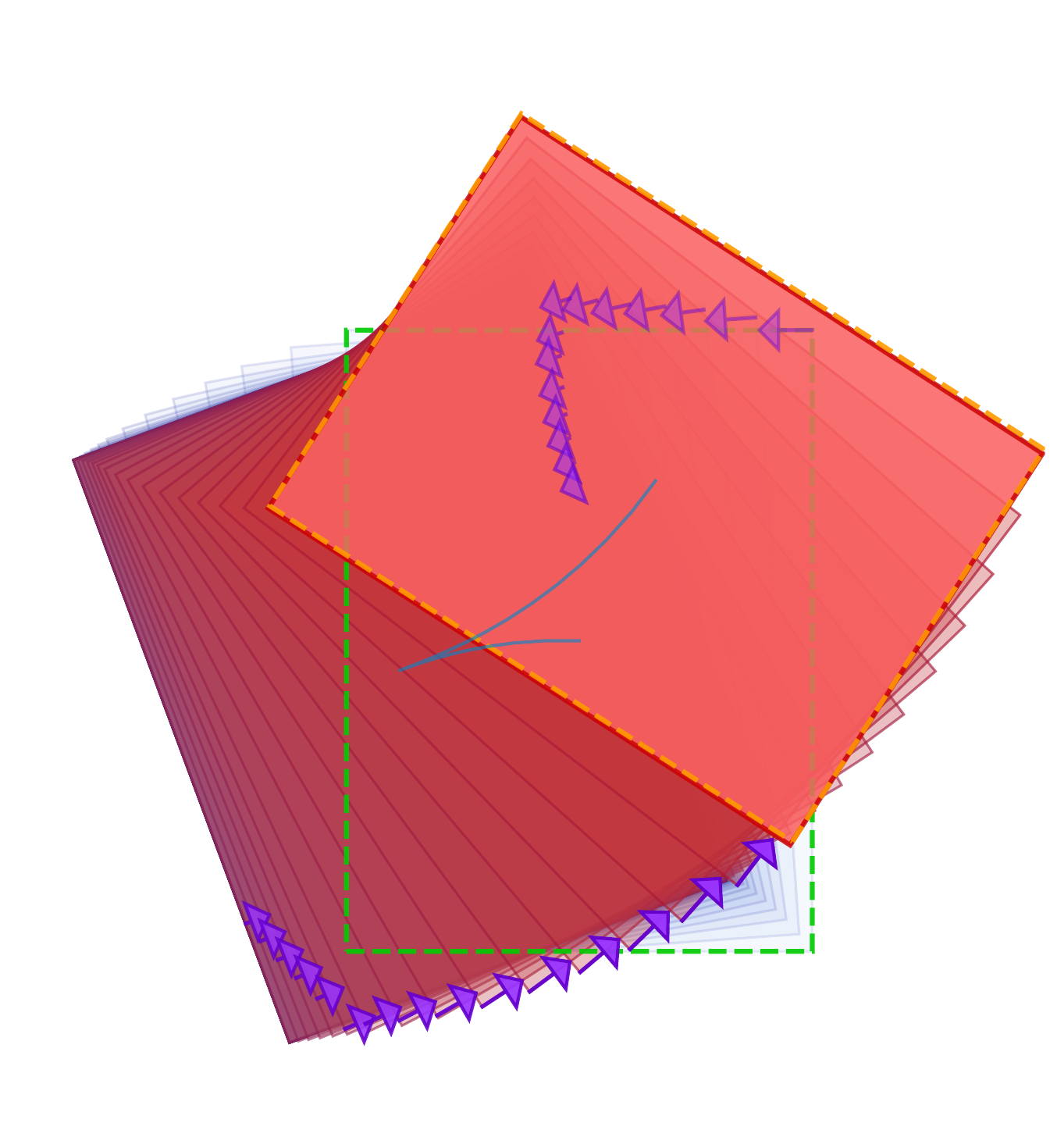}
    \caption{From $(x,y,\theta)=(0.00,0.00,0.00)$ to $(0.10,0.21,1.00)$.}
    \label{fig:box-a}
  \end{subfigure}\hspace{0.1\linewidth}
  \begin{subfigure}[t]{0.32\linewidth}
    \centering
    \includegraphics[width=\linewidth]{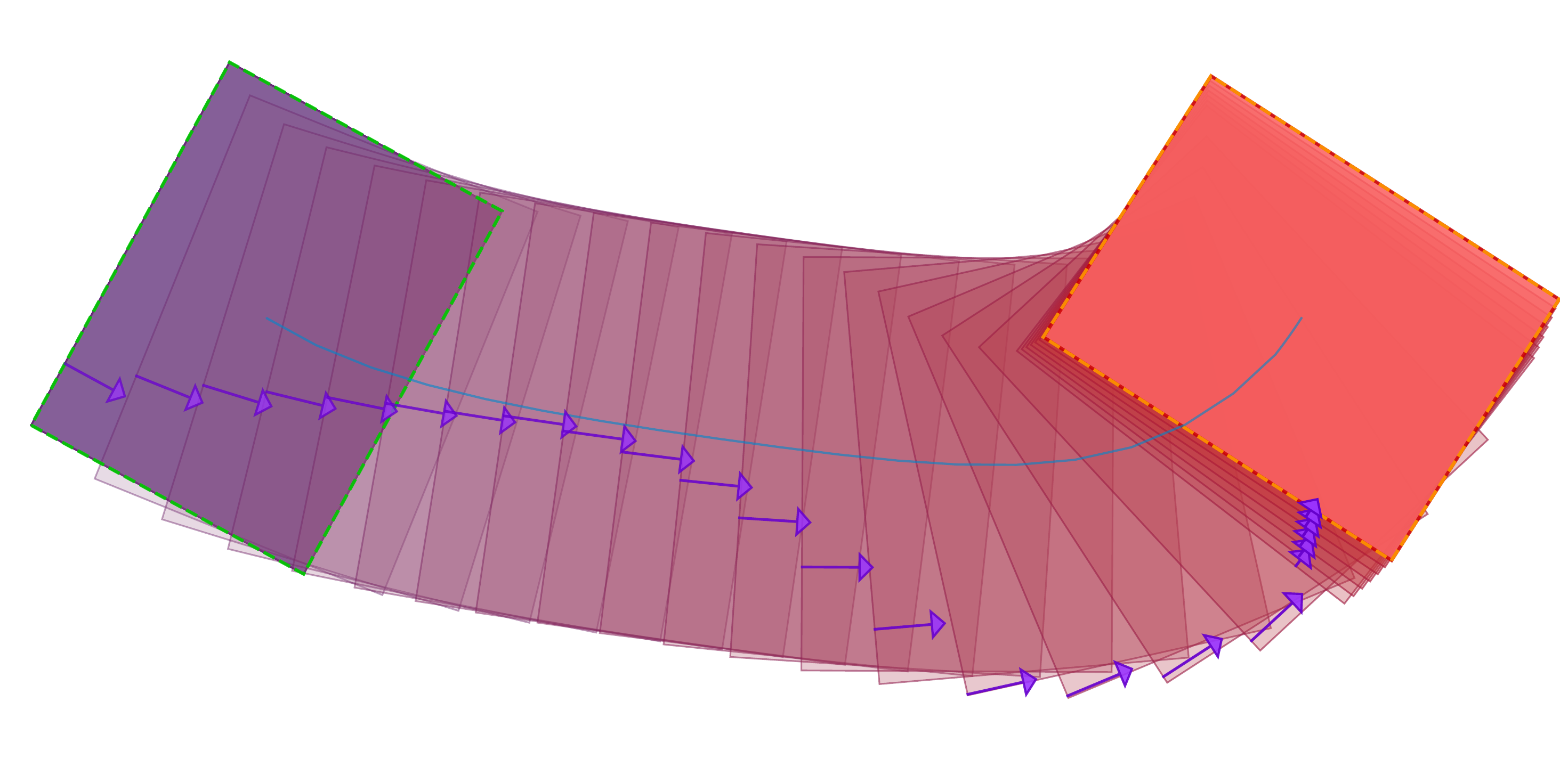}
    \caption{From $(x,y,\theta)=(0.00,0.00,-0.50)$ to $(2.00,0.00,1.00)$.}
    \label{fig:box-b}
  \end{subfigure}

  \begin{subfigure}[t]{0.3\linewidth}
    \centering
    \includegraphics[width=\linewidth]{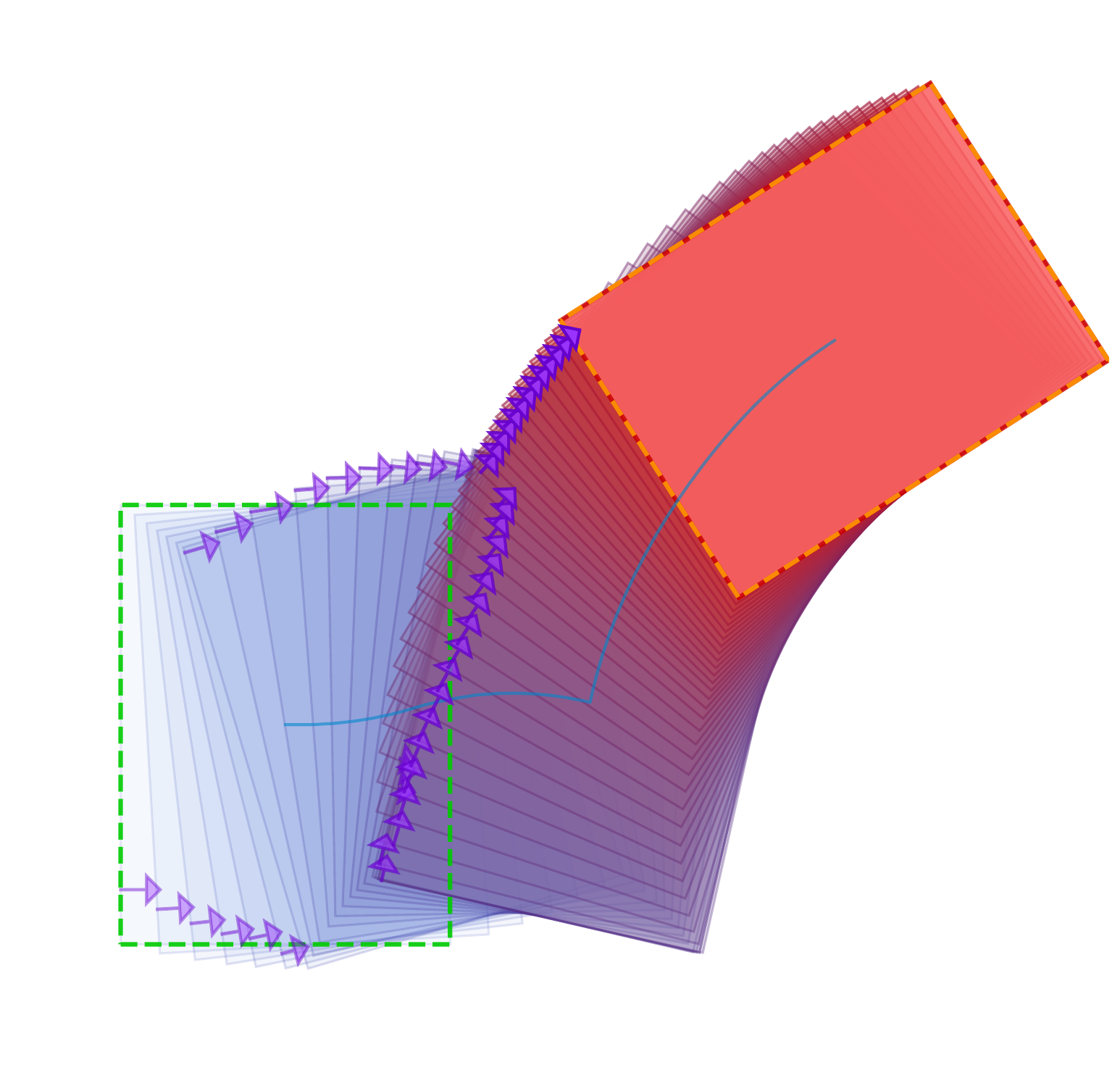}
    \caption{From $(x,y,\theta)=(0.00,0.00,0.00)$ to $(1.00,0.70,-1.00)$.}
    \label{fig:box-c}
  \end{subfigure}\hspace{0.1\linewidth}
  \begin{subfigure}[t]{0.3\linewidth}
    \centering
    \includegraphics[width=\linewidth]{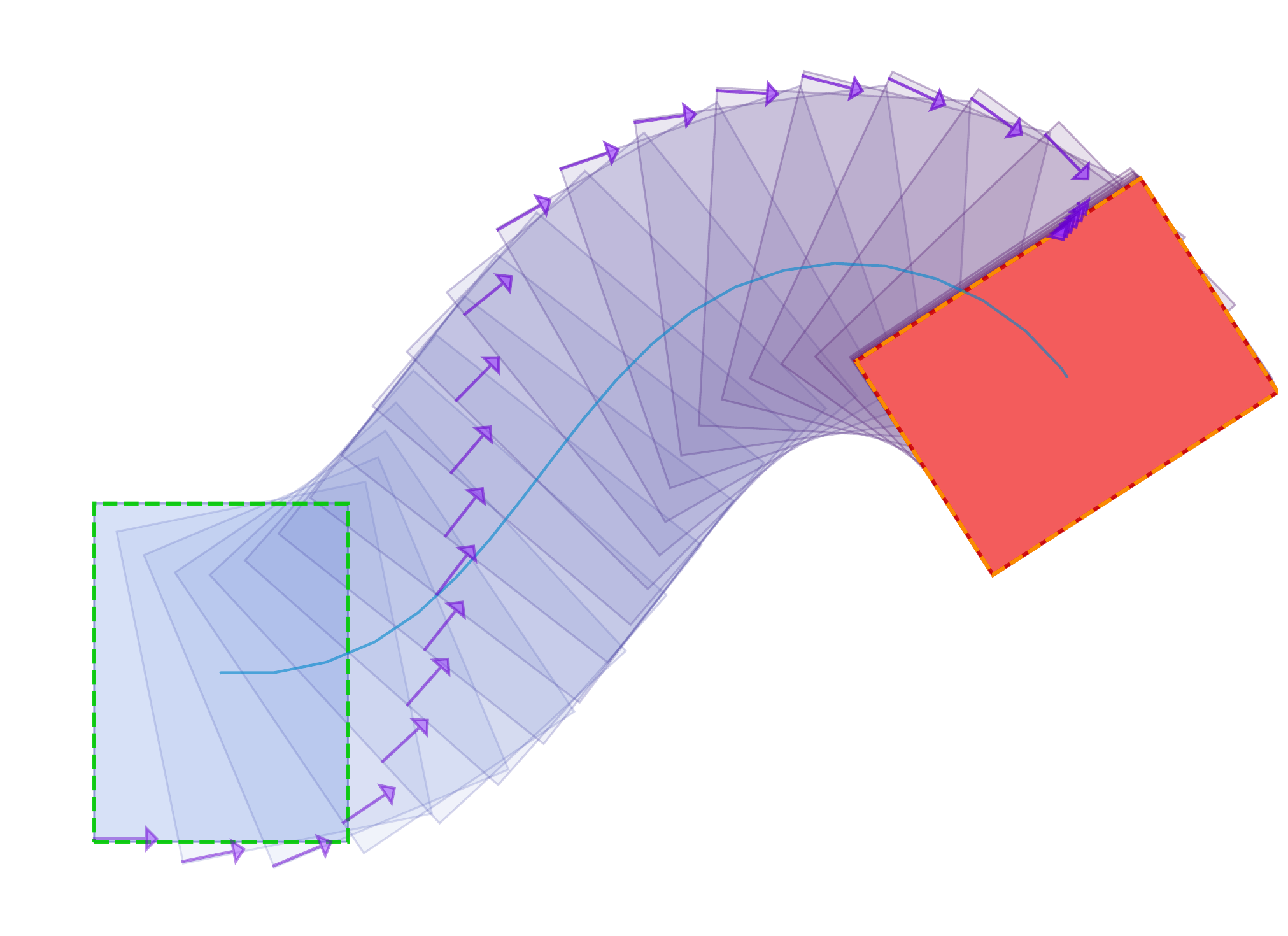}
    \caption{From $(x,y,\theta)=(0.00,0.00,0.00)$ to $(2.00,0.70,-1.00)$.}
    \label{fig:box-d}
  \end{subfigure}

  \caption{Pushing Box task. Each panel overlays one rollout conditioned on a start and a goal state in $(x, y, \theta)$. The green and orange dashed boxes indicate the start and goal poses, respectively. The purple arrow shows the pushing force, and the color gradient from light blue to orange indicates time progression. The blue solid line shows the trajectory of the box center.}
  \label{fig:box-2x2}
\end{figure}

\begin{figure}[t]
  \centering

  \begin{subfigure}[t]{0.3\linewidth}
    \centering
    \includegraphics[width=\linewidth]{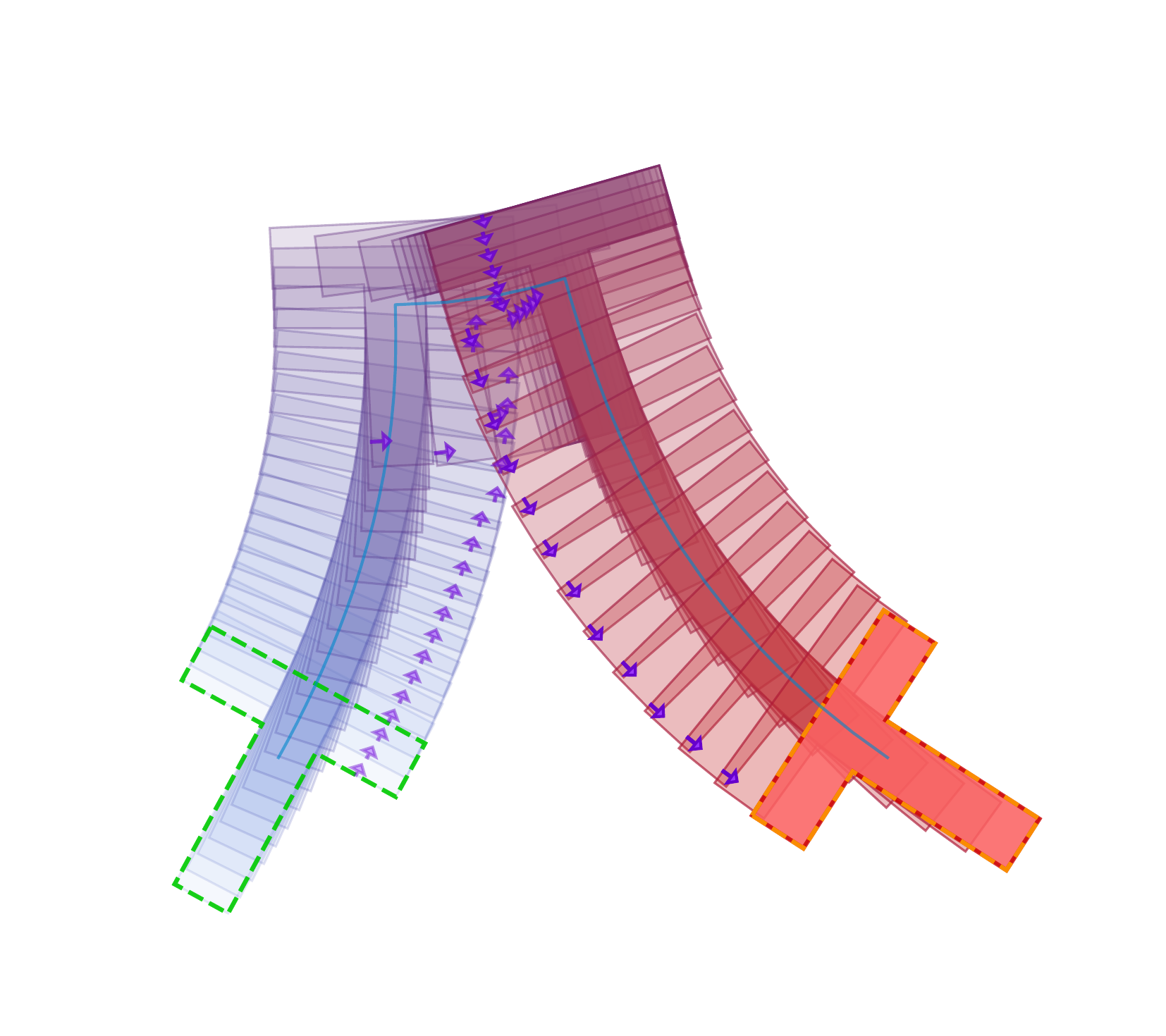}
    \caption{From $(x,y,\theta)=(0.00,0.00,-0.50)$ to $(0.50,0.00,1.00)$.}
    \label{fig:pusht-a}
  \end{subfigure}\hspace{0.04\linewidth}
  \begin{subfigure}[t]{0.4\linewidth}
    \centering
    \includegraphics[width=\linewidth]{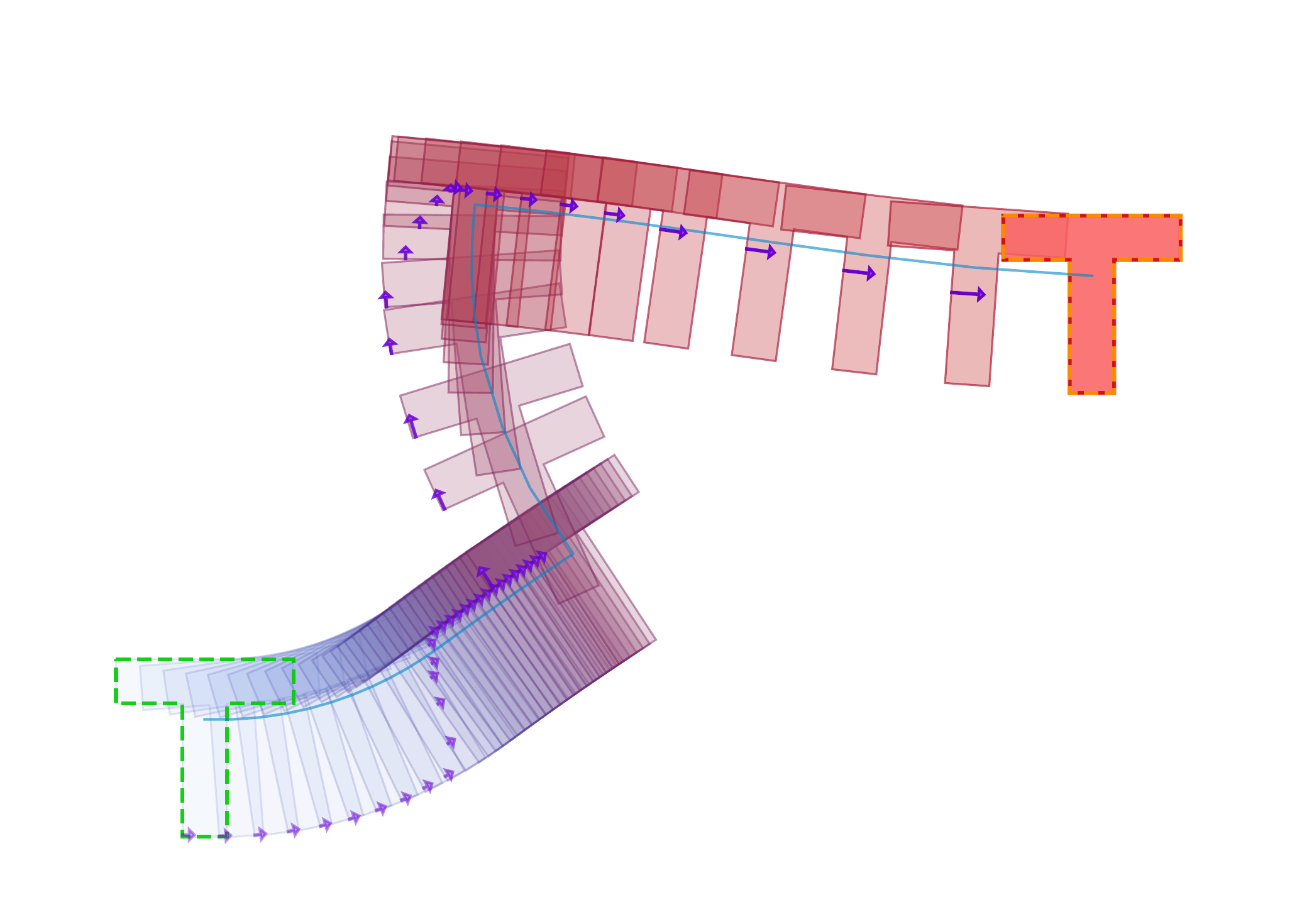}
    \caption{From $(x,y,\theta)=(0.00,0.00,0.00)$ to $(1.00,0.50,0.00)$.}
    \label{fig:pusht-b}
  \end{subfigure}

  \begin{subfigure}[t]{0.4\linewidth}
    \centering
    \includegraphics[width=\linewidth]{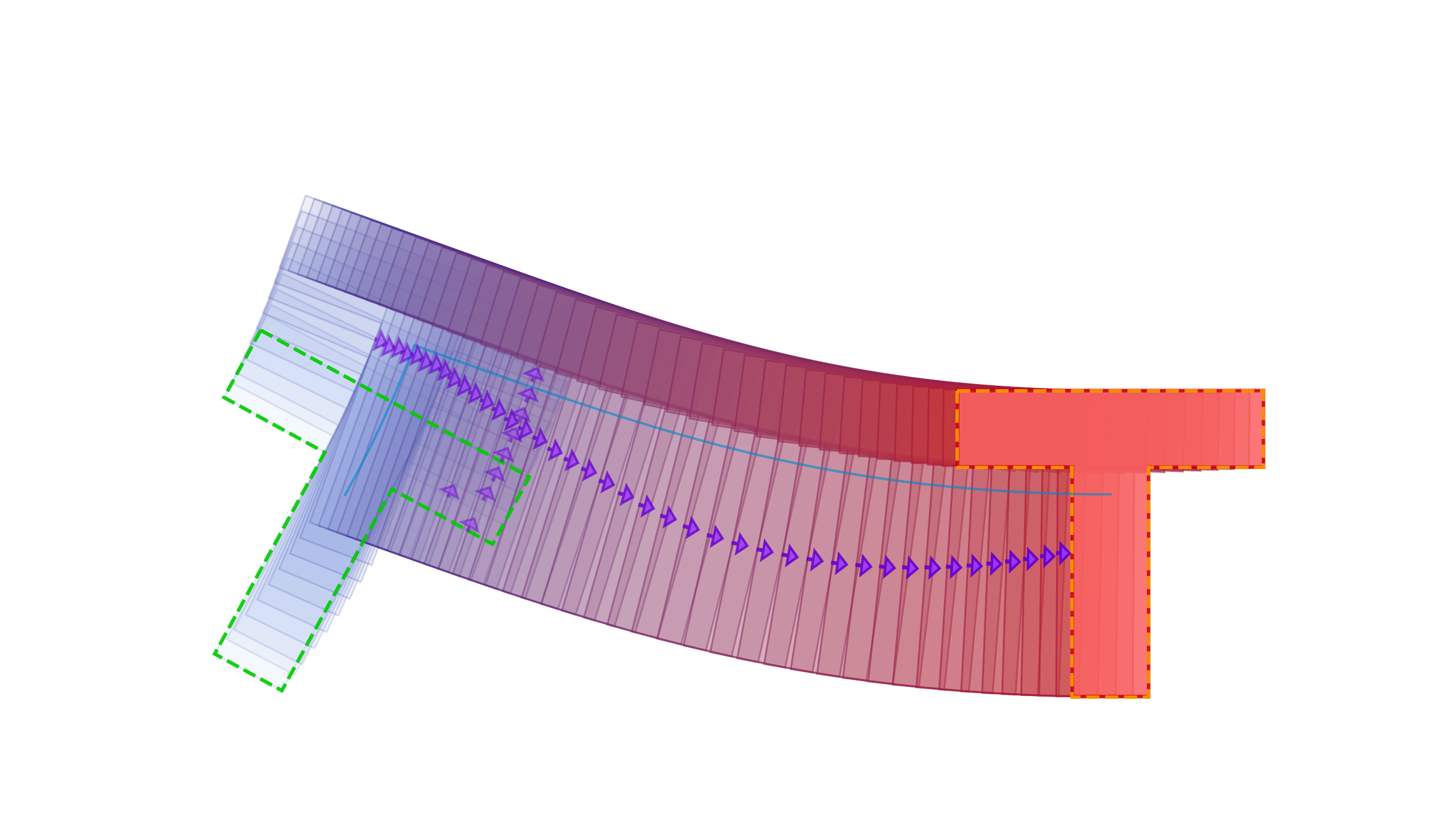}
    \caption{From $(x,y,\theta)=(0.00,0.00,-0.50)$ to $(0.50,0.00,0.00)$.}
    \label{fig:pusht-c}
  \end{subfigure}\hspace{0.04\linewidth}
  \begin{subfigure}[t]{0.4\linewidth}
    \centering
    \includegraphics[width=\linewidth]{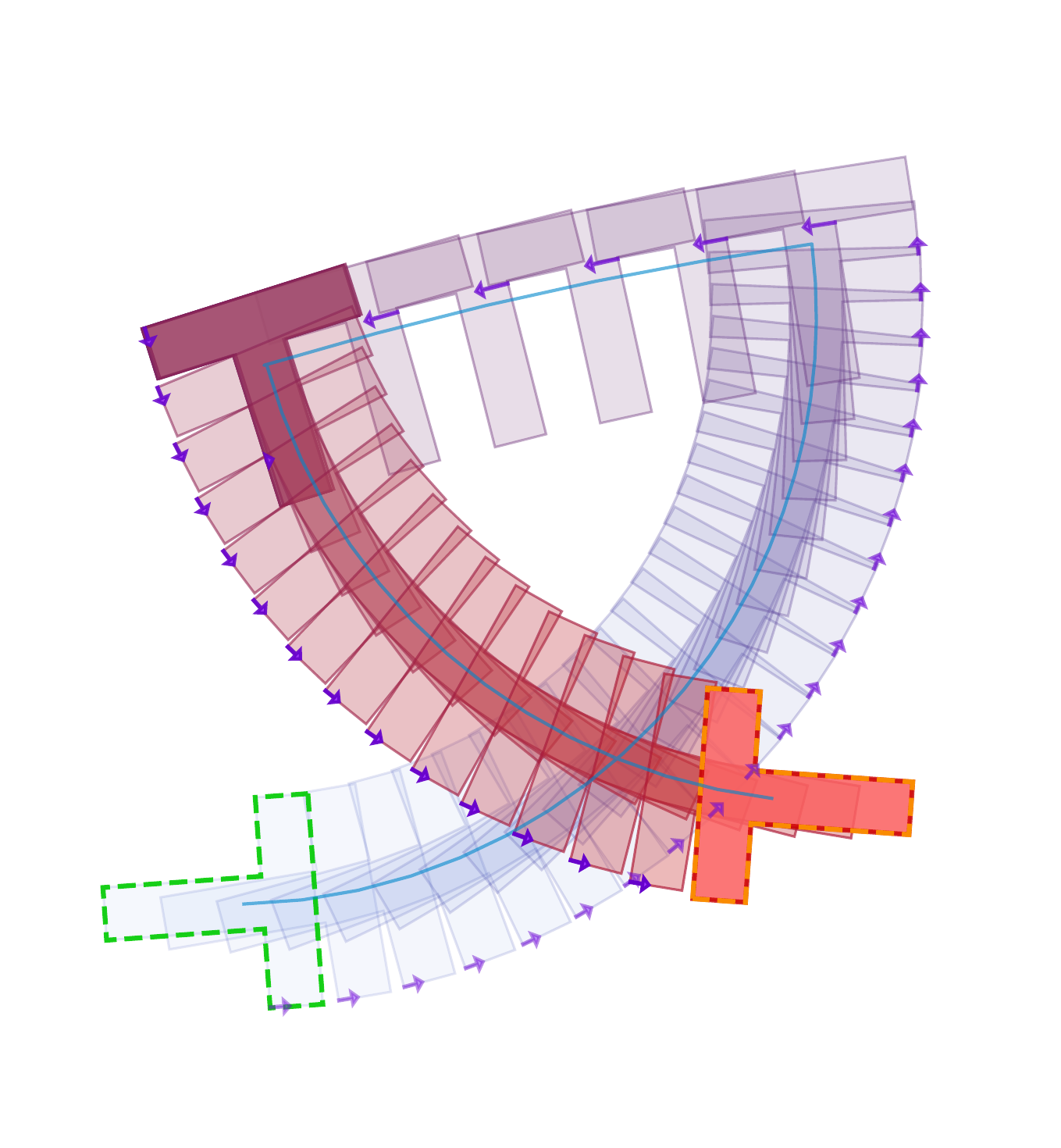}
    \caption{From $(x,y,\theta)=(0.00,0.00,-1.50)$ to $(0.50,0.10,1.50)$.}
    \label{fig:pusht-d}
  \end{subfigure}

  \caption{Push T task. Each panel overlays one rollout conditioned on a start and a goal state in $(x, y, \theta)$. The green and orange dashed boxes indicate the start and goal poses, respectively. The purple arrow shows the pushing force, and the color gradient from light blue to orange indicates time progression. The blue solid line shows the trajectory of the pushed object's center.}
  \label{fig:pusht-2x2}
\end{figure}

\begin{figure}[htbp]
  \centering
  \begin{subfigure}[t]{0.211\textwidth}
    \centering
    \includegraphics[width=\linewidth]{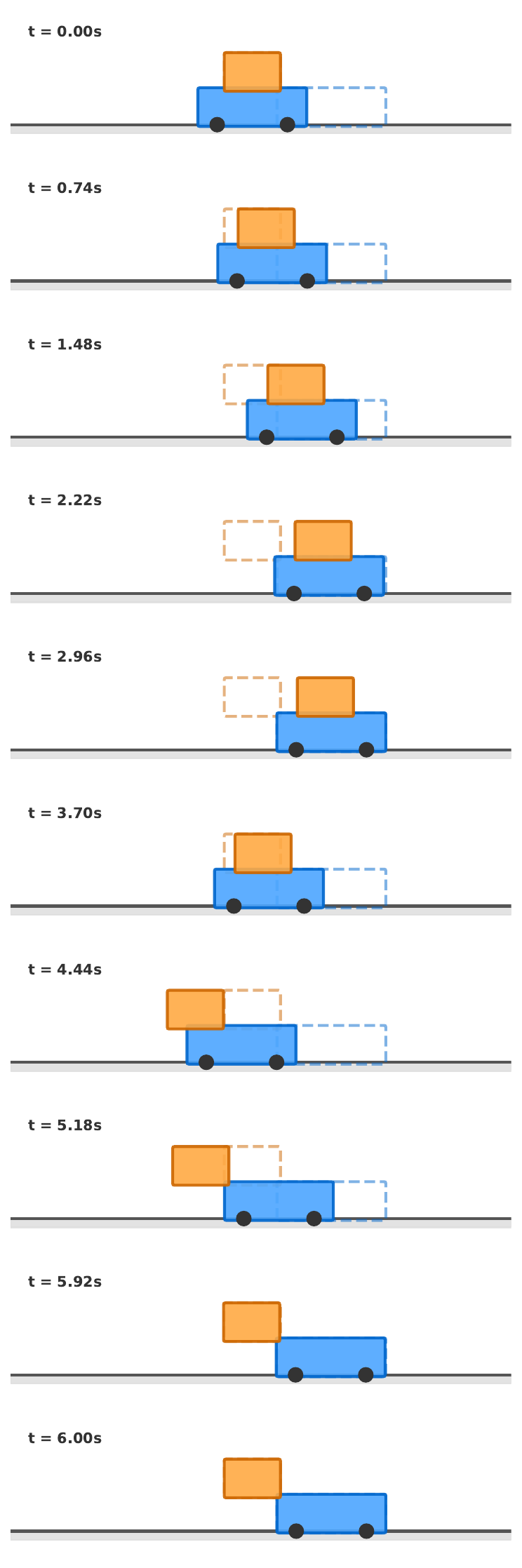}
    \caption{Start at x=0.00, y=0.00. Goal is x=0.00, y=0.90.}
    \label{fig:cart-a}
  \end{subfigure}
  \begin{subfigure}[t]{0.2\textwidth}
    \centering
    \includegraphics[width=\linewidth]{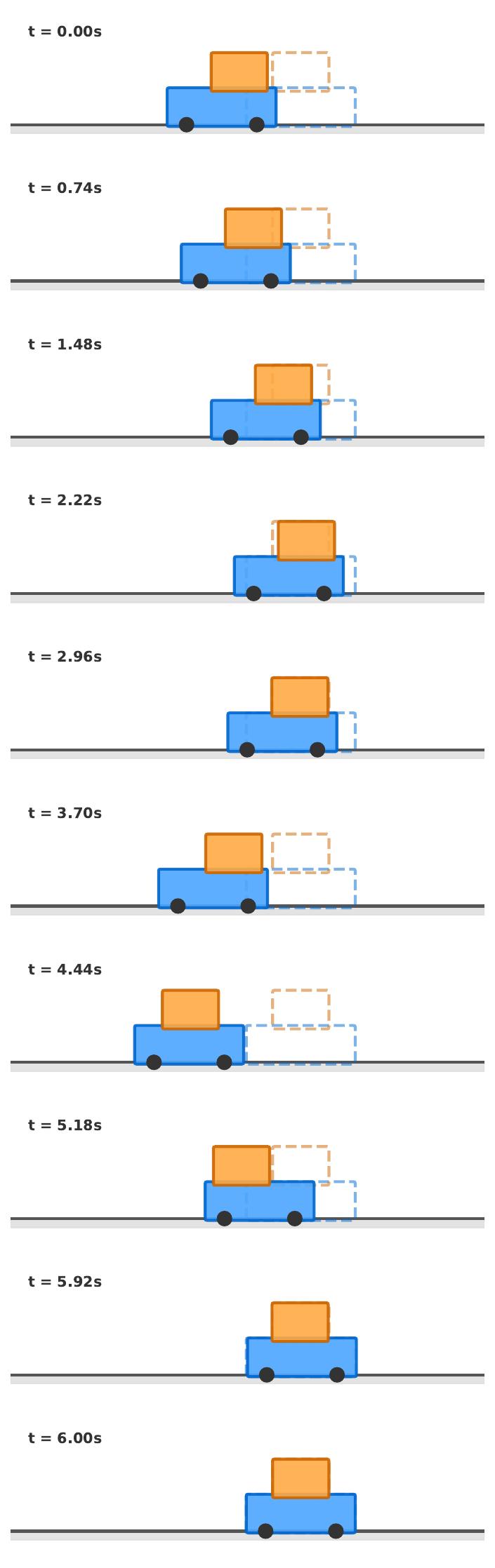}
    \caption{Start at x=0.00, y=-0.20. Goal is x=0.70, y=0.70.}
    \label{fig:cart-b}
  \end{subfigure}
  \begin{subfigure}[t]{0.257\textwidth}
    \centering
    \includegraphics[width=\linewidth]{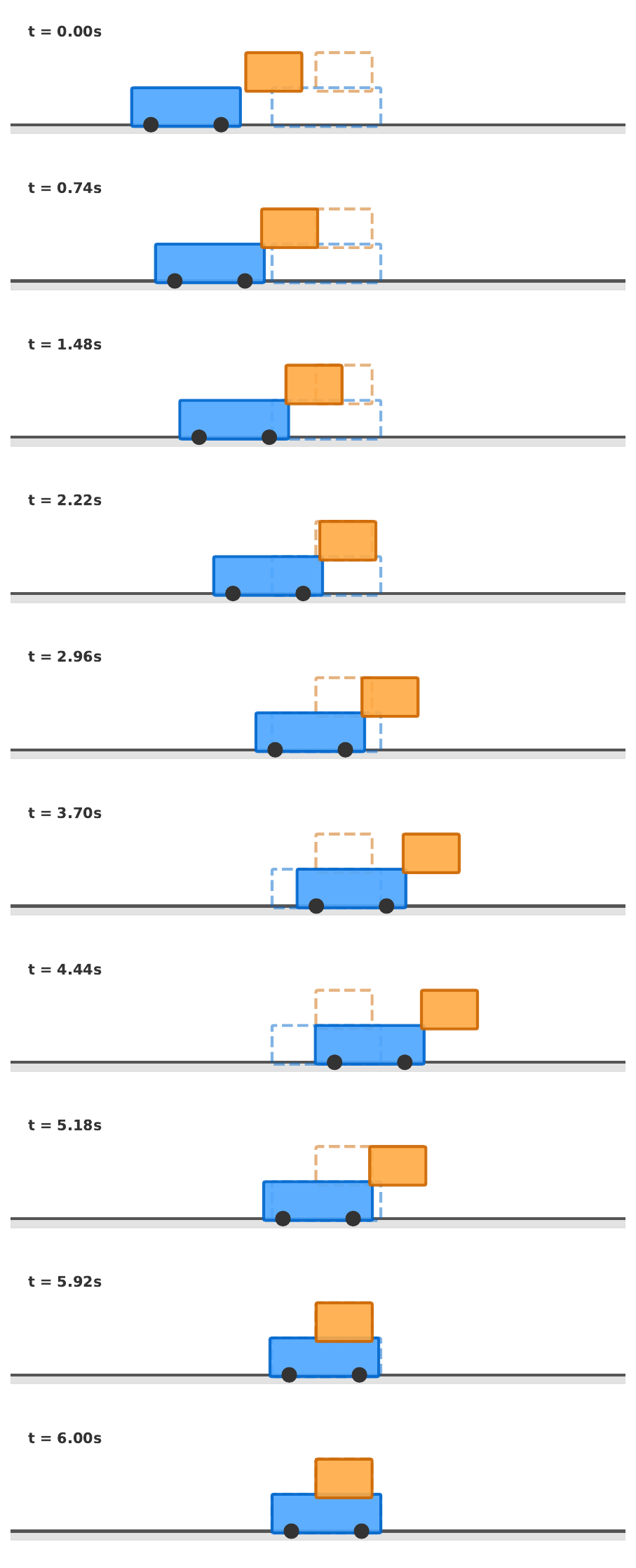}
    \caption{Start at x=0.90, y=-0.10. Goal is x=1.70, y=1.50.}
    \label{fig:cart-c}
  \end{subfigure}
  \begin{subfigure}[t]{0.239\textwidth}
    \centering
    \includegraphics[width=\linewidth]{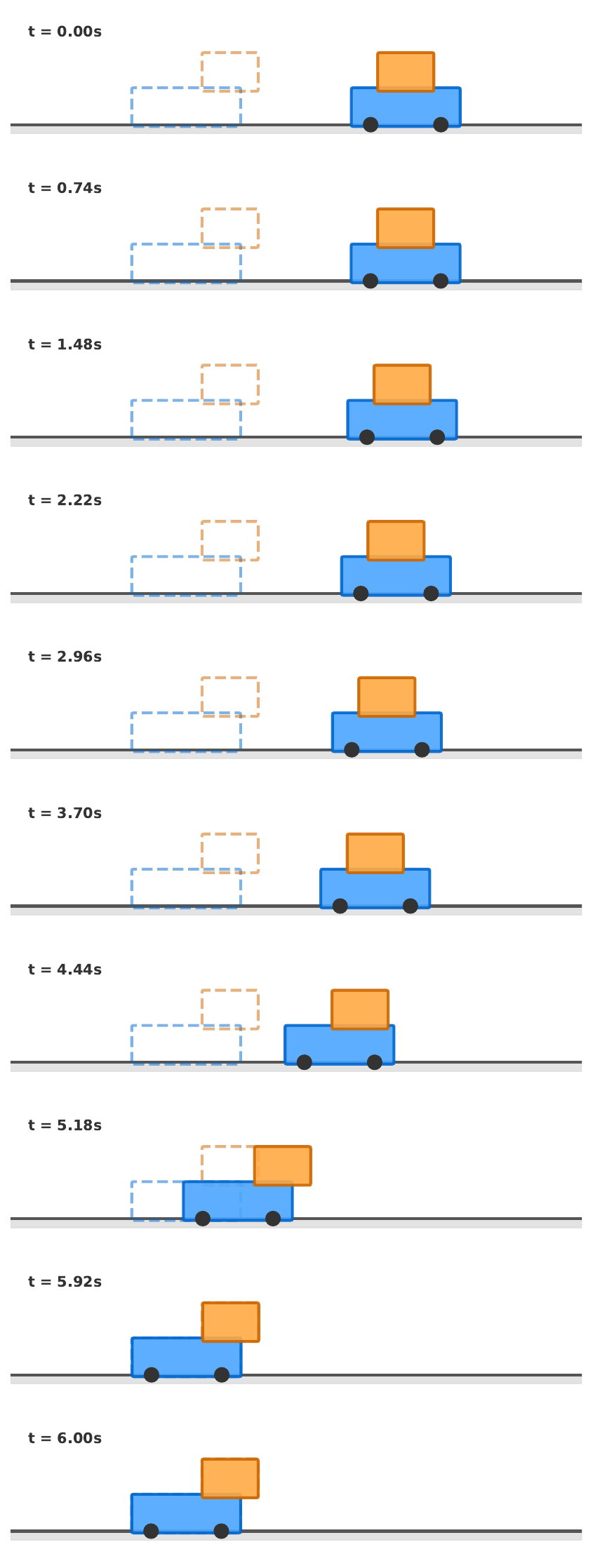}
    \caption{Start at x=1.00, y=1.00. Goal is x=-1.00, y=-1.50.}
    \label{fig:cart-d}
  \end{subfigure}

  \caption{Cart Transporter task. Each panel shows a rollout sequence in the state space $(x, y, dx, dy)$, where $x$ and $dx$ denote the load's position and velocity, and $y$ and $dy$ denote the cart's position and velocity, all expressed in the global coordinate frame. Start and goal velocities are set to zero.}
  \label{fig:cart-1x4}
\end{figure}

\subsection{Allegro Extra Demos}
We provide additional Allegro in-hand re-orientation results on a variety of objects.
Due to space constraints, we split them into four figures containing 5, 5, 5, and 2 objects, respectively
(Figs.~\ref{fig:allegro_extra_1}--\ref{fig:allegro_extra_4}). The object shown above the hand indicates the target pose.

\begin{figure}[t]
  \centering

  \begin{subfigure}{\linewidth}
    \centering
    \includegraphics[width=0.8\linewidth]{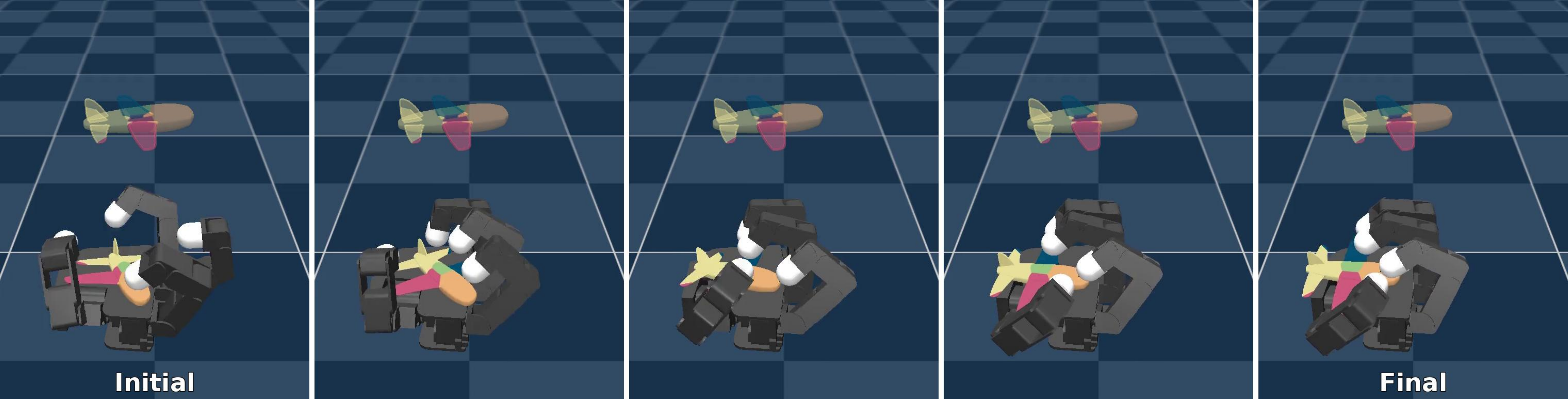}
    \caption{Airplane}
    \label{fig:allegro_airplane}
  \end{subfigure}

  \vspace{0.6em}

  \begin{subfigure}{\linewidth}
    \centering
    \includegraphics[width=0.8\linewidth]{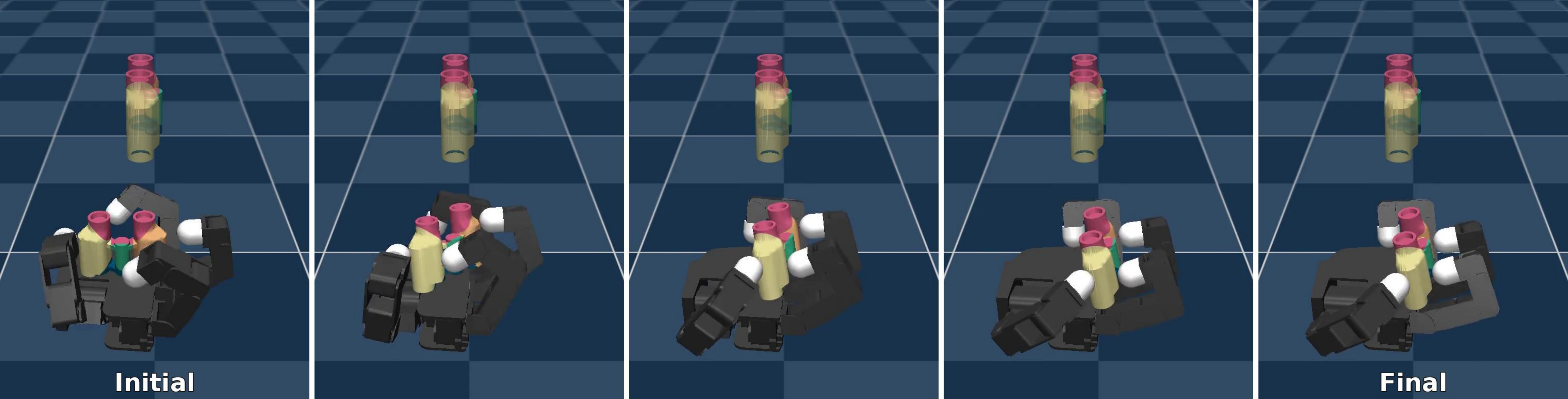}
    \caption{Binoculars}
    \label{fig:allegro_binoculars}
  \end{subfigure}

  \vspace{0.6em}

  \begin{subfigure}{\linewidth}
    \centering
    \includegraphics[width=0.8\linewidth]{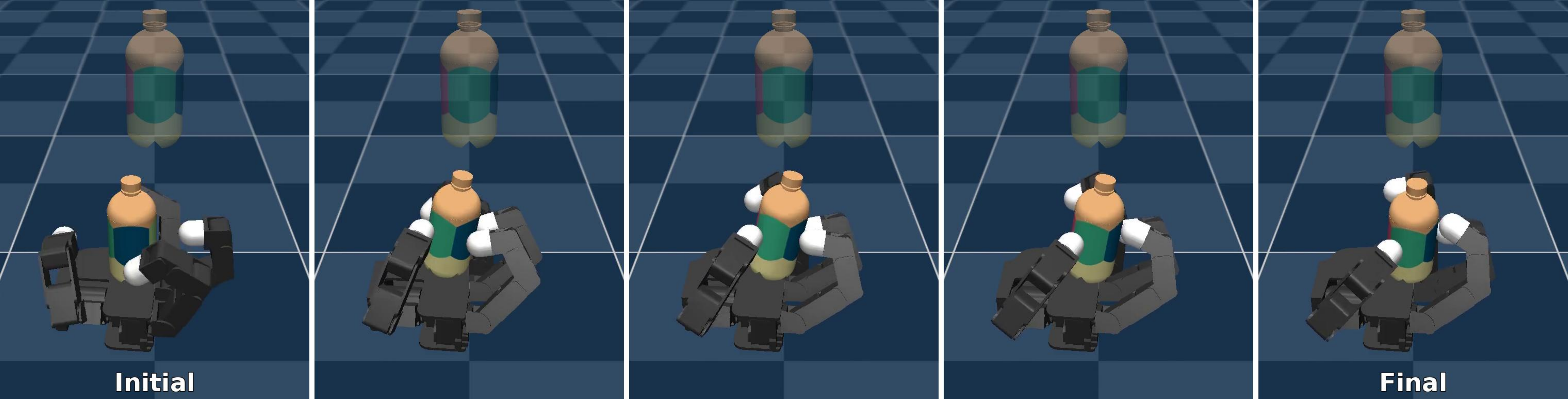}
    \caption{Bottle}
    \label{fig:allegro_bottle}
  \end{subfigure}

  \vspace{0.6em}

  \begin{subfigure}{\linewidth}
    \centering
    \includegraphics[width=0.8\linewidth]{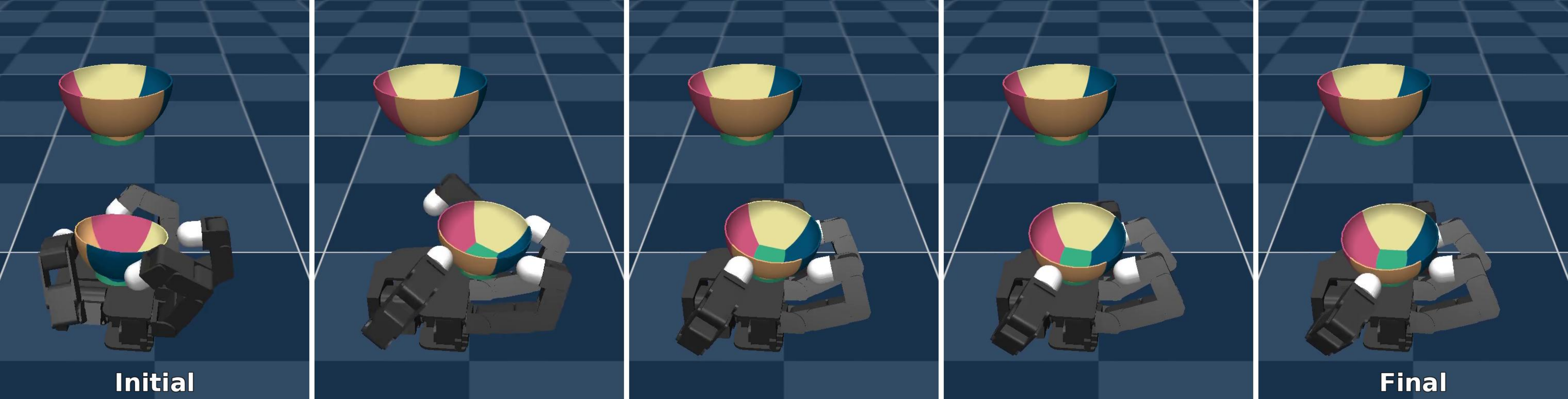}
    \caption{Bowl}
    \label{fig:allegro_bowl}
  \end{subfigure}

  \vspace{0.6em}

  \begin{subfigure}{\linewidth}
    \centering
    \includegraphics[width=0.8\linewidth]{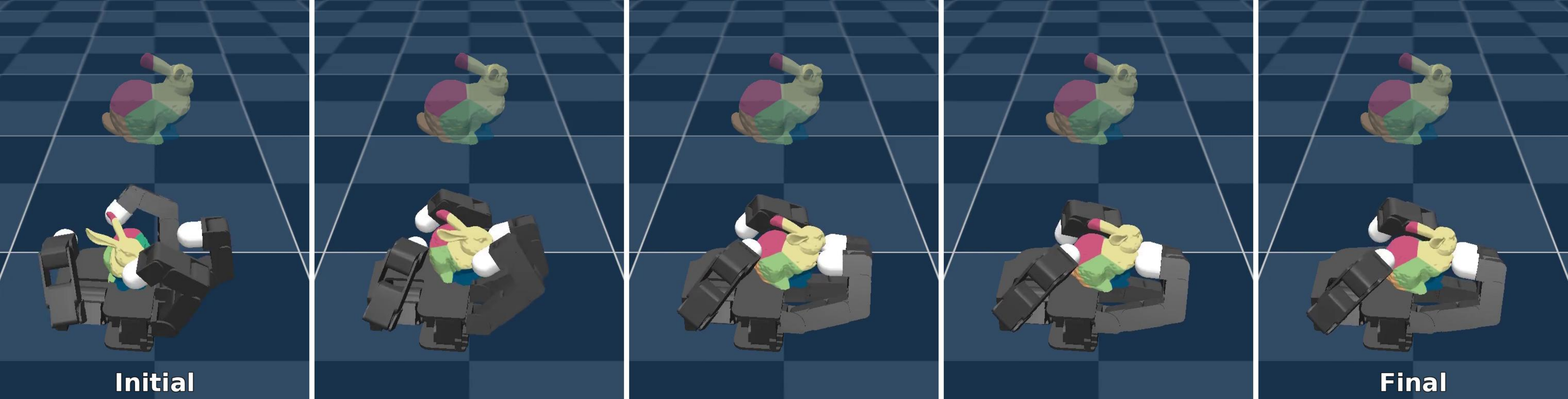}
    \caption{Bunny}
    \label{fig:allegro_bunny}
  \end{subfigure}

  \caption{Additional Allegro in-hand re-orientation results on: Airplane, Binoculars, Bottle, Bowl, and Bunny.}
  \label{fig:allegro_extra_1}
\end{figure}

\begin{figure}[t]
  \centering

  \begin{subfigure}{\linewidth}
    \centering
    \includegraphics[width=0.8\linewidth]{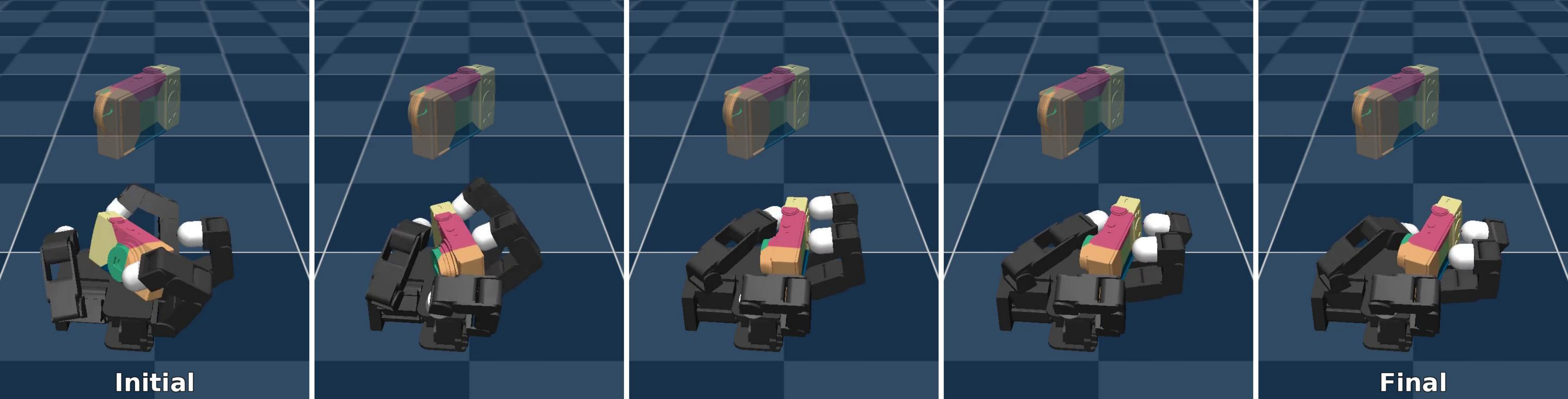}
    \caption{Camera}
    \label{fig:allegro_camera}
  \end{subfigure}

  \vspace{0.6em}

  \begin{subfigure}{\linewidth}
    \centering
    \includegraphics[width=0.8\linewidth]{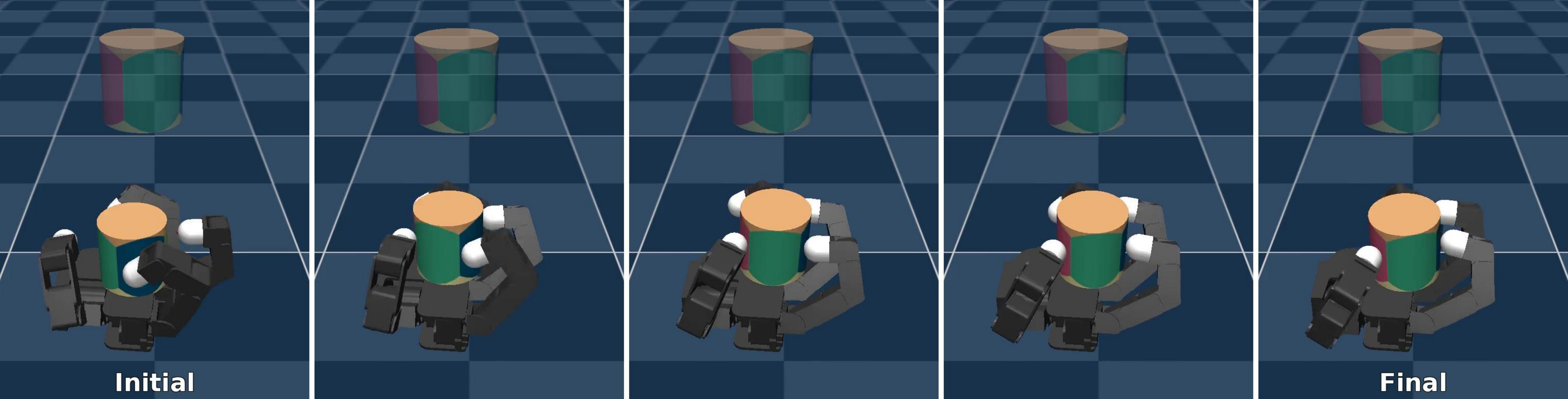}
    \caption{Can}
    \label{fig:allegro_can}
  \end{subfigure}

  \vspace{0.6em}

  \begin{subfigure}{\linewidth}
    \centering
    \includegraphics[width=0.8\linewidth]{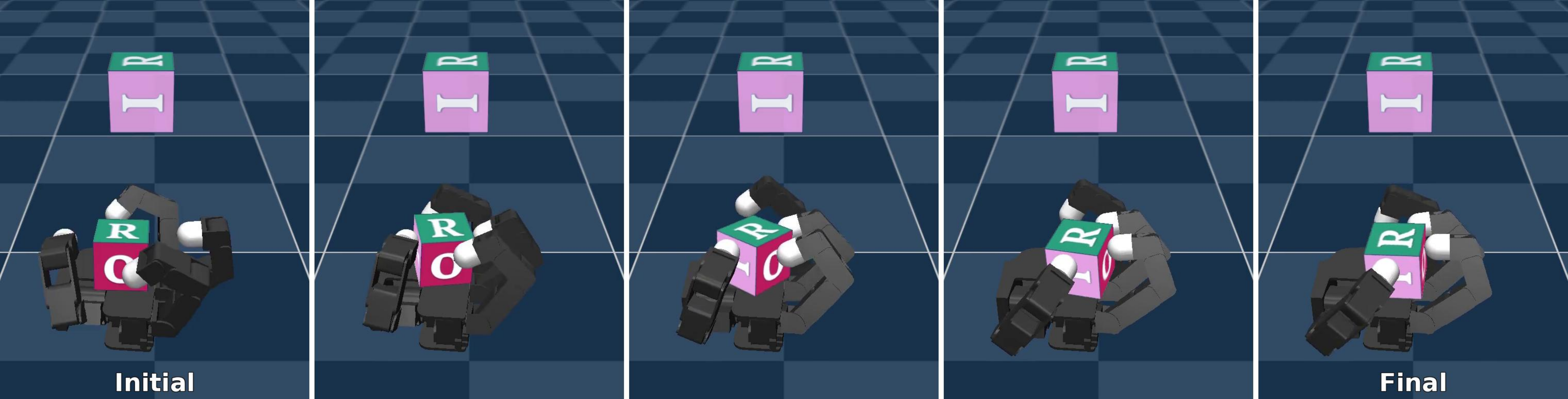}
    \caption{Cube}
    \label{fig:allegro_cube}
  \end{subfigure}

  \vspace{0.6em}

  \begin{subfigure}{\linewidth}
    \centering
    \includegraphics[width=0.8\linewidth]{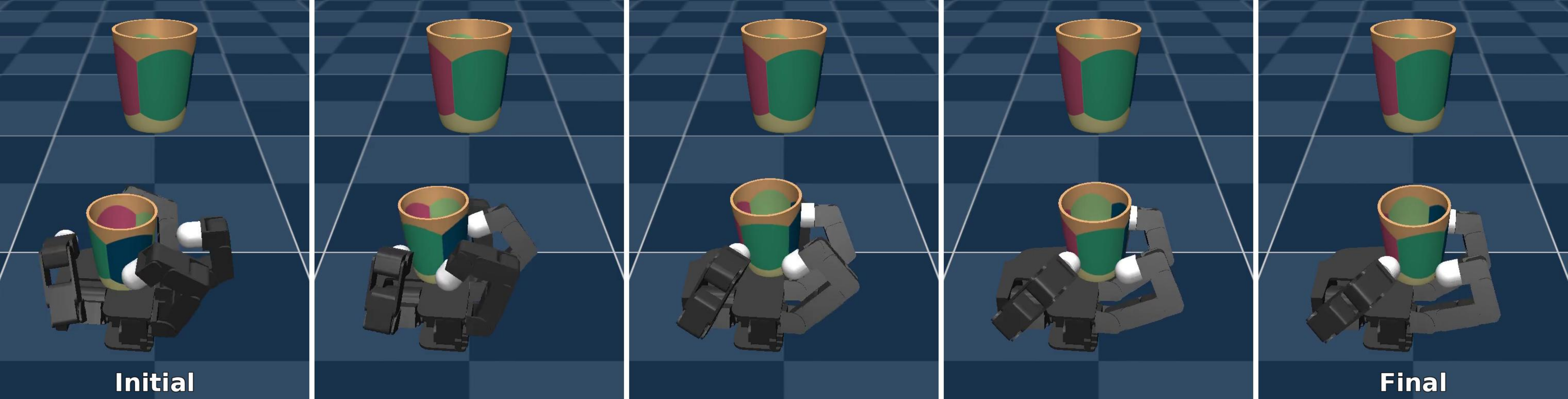}
    \caption{Cup}
    \label{fig:allegro_cup}
  \end{subfigure}

  \vspace{0.6em}

  \begin{subfigure}{\linewidth}
    \centering
    \includegraphics[width=0.8\linewidth]{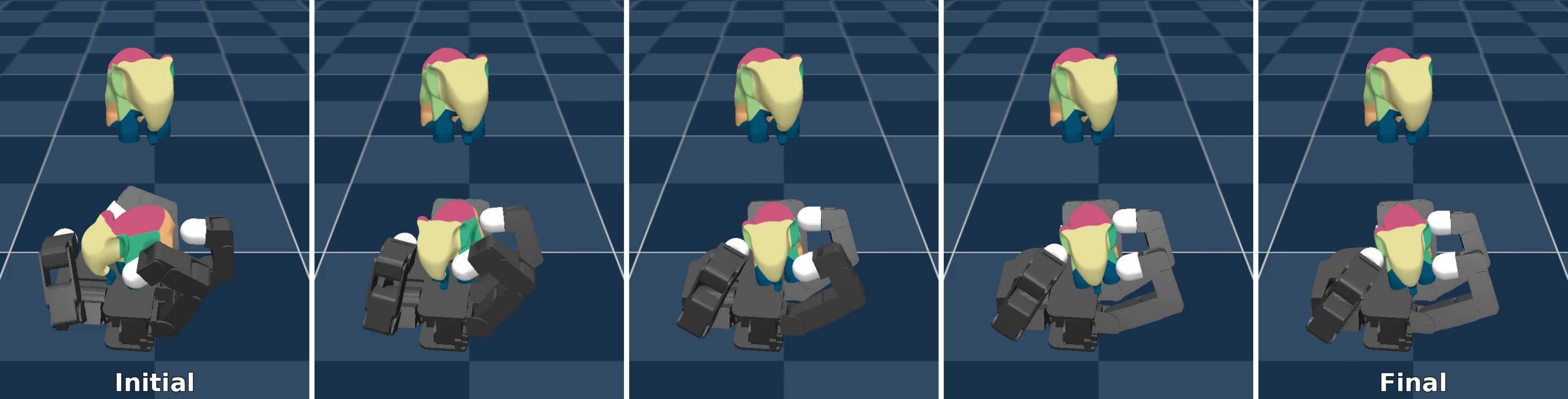}
    \caption{Elephant}
    \label{fig:allegro_elephant}
  \end{subfigure}

  \caption{Additional Allegro in-hand re-orientation results on: Camera, Can, Cube, Cup, and Elephant.}
  \label{fig:allegro_extra_2}
\end{figure}

\begin{figure}[t]
  \centering

  \begin{subfigure}{\linewidth}
    \centering
    \includegraphics[width=0.8\linewidth]{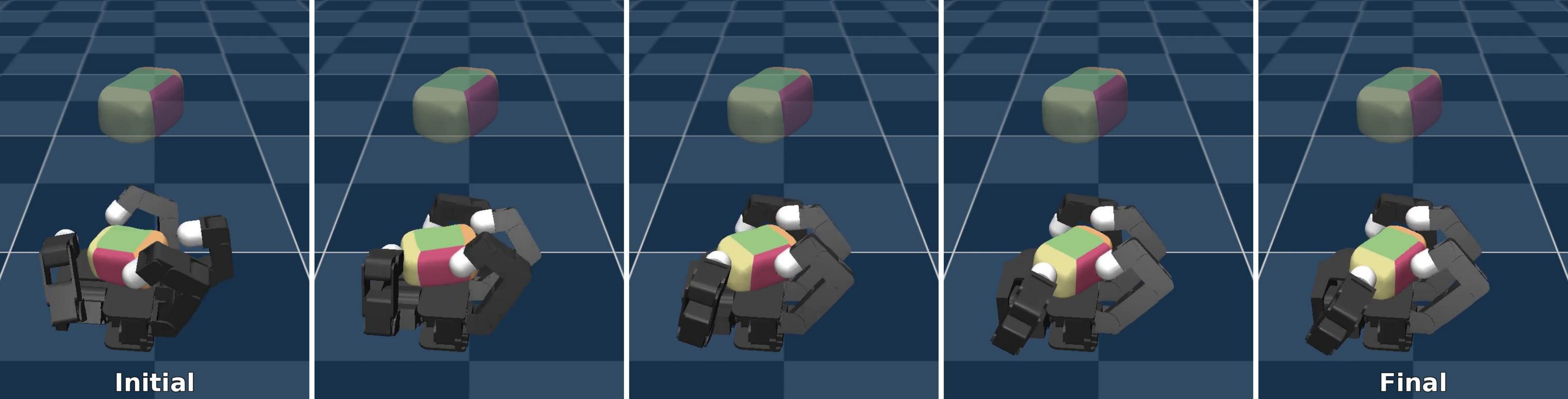}
    \caption{Foam brick}
    \label{fig:allegro_foambrick}
  \end{subfigure}

  \vspace{0.6em}

  \begin{subfigure}{\linewidth}
    \centering
    \includegraphics[width=0.8\linewidth]{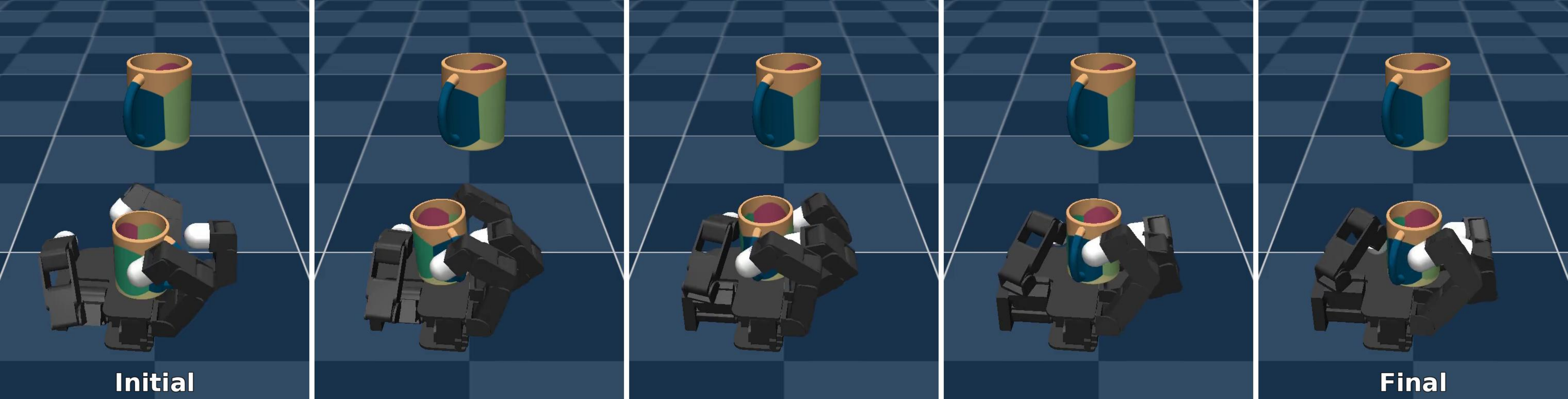}
    \caption{Mug}
    \label{fig:allegro_mug}
  \end{subfigure}

  \vspace{0.6em}

  \begin{subfigure}{\linewidth}
    \centering
    \includegraphics[width=0.8\linewidth]{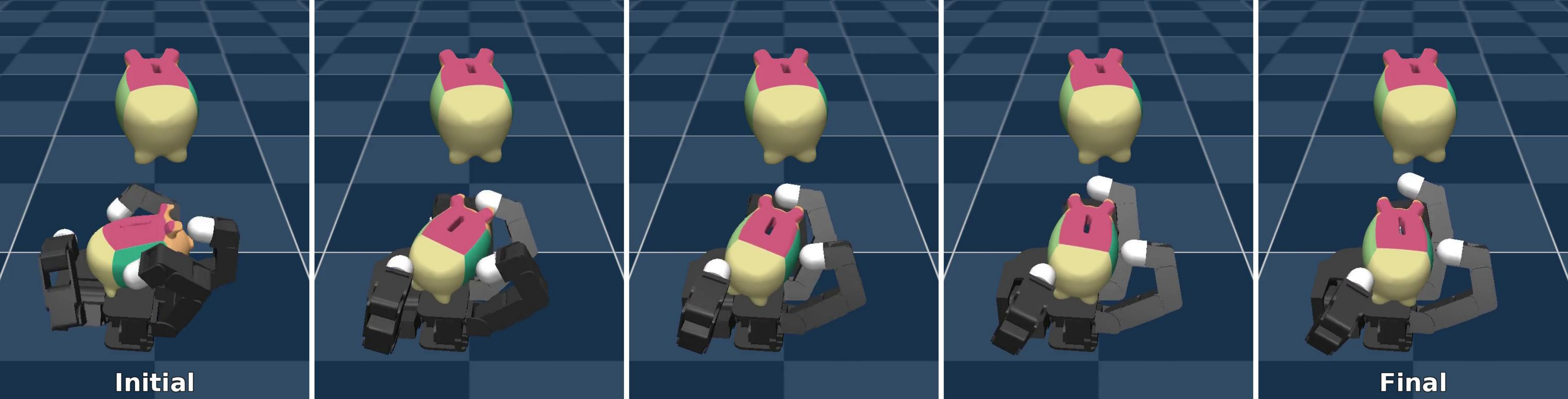}
    \caption{Piggy}
    \label{fig:allegro_piggy}
  \end{subfigure}

  \vspace{0.6em}

  \begin{subfigure}{\linewidth}
    \centering
    \includegraphics[width=0.8\linewidth]{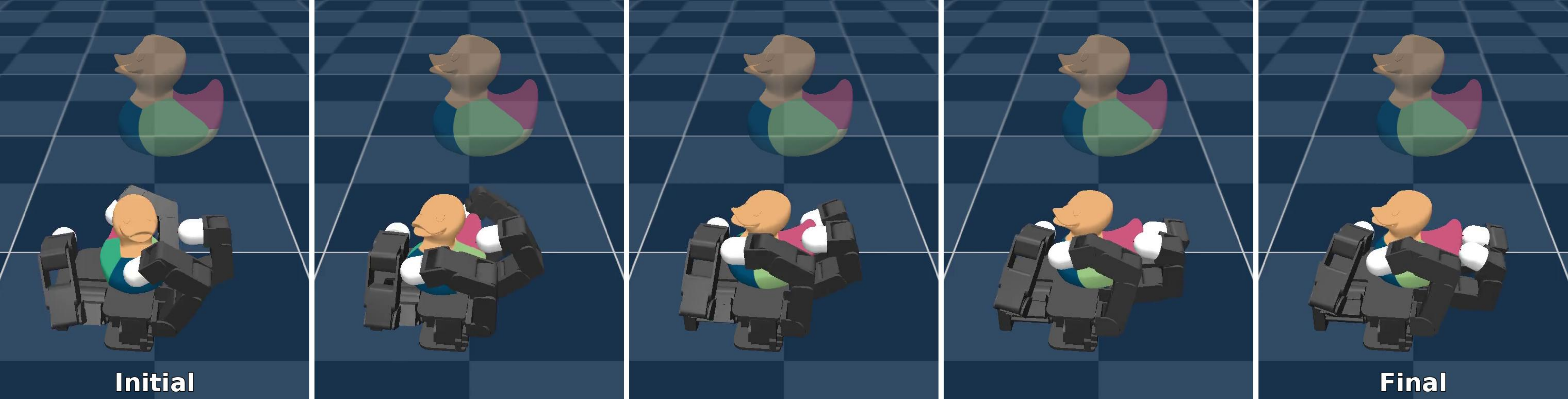}
    \caption{Rubber object}
    \label{fig:allegro_rubber}
  \end{subfigure}

  \vspace{0.6em}

  \begin{subfigure}{\linewidth}
    \centering
    \includegraphics[width=0.8\linewidth]{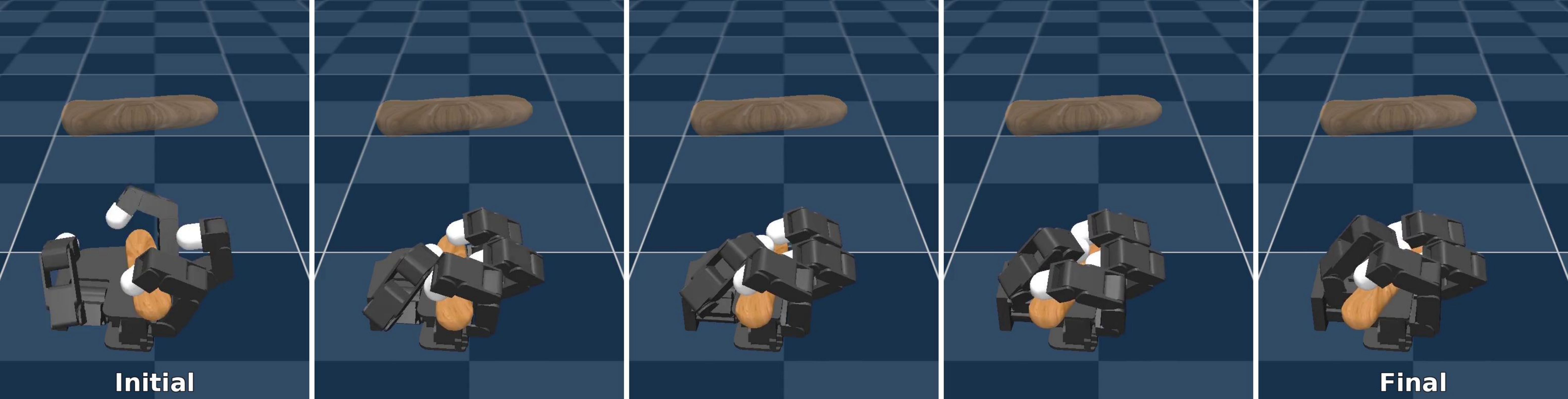}
    \caption{Stick}
    \label{fig:allegro_stick}
  \end{subfigure}

  \caption{Additional Allegro in-hand re-orientation results on: Foam brick, Mug, Piggy, Rubber object, and Stick (fail example).}
  \label{fig:allegro_extra_3}
\end{figure}

\begin{figure}[t]
  \centering

  \begin{subfigure}{\linewidth}
    \centering
    \includegraphics[width=0.8\linewidth]{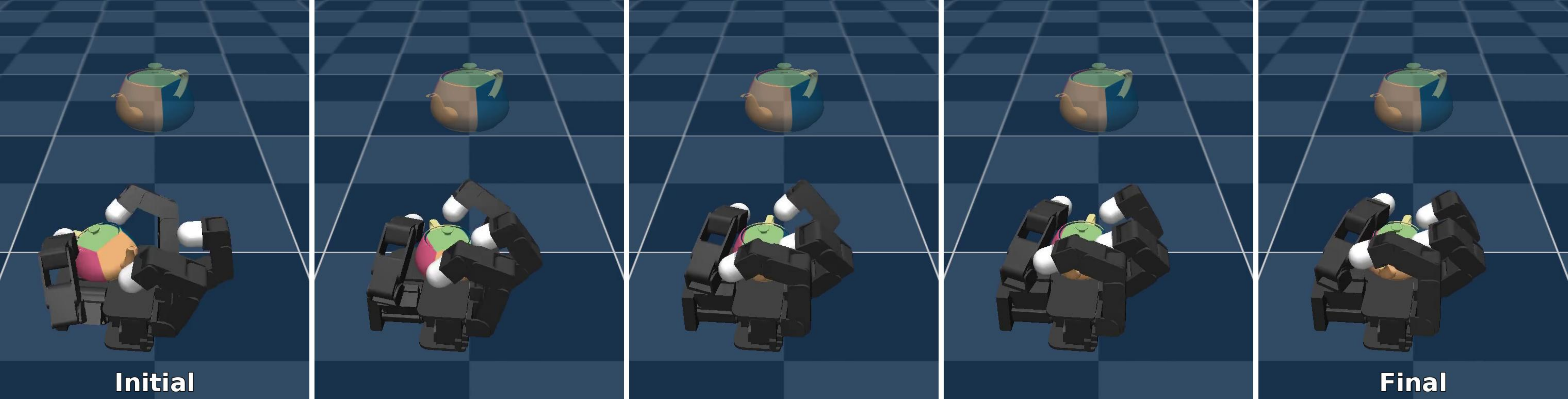}
    \caption{Teapot}
    \label{fig:allegro_teapot}
  \end{subfigure}

  \vspace{0.6em}

  \begin{subfigure}{\linewidth}
    \centering
    \includegraphics[width=0.8\linewidth]{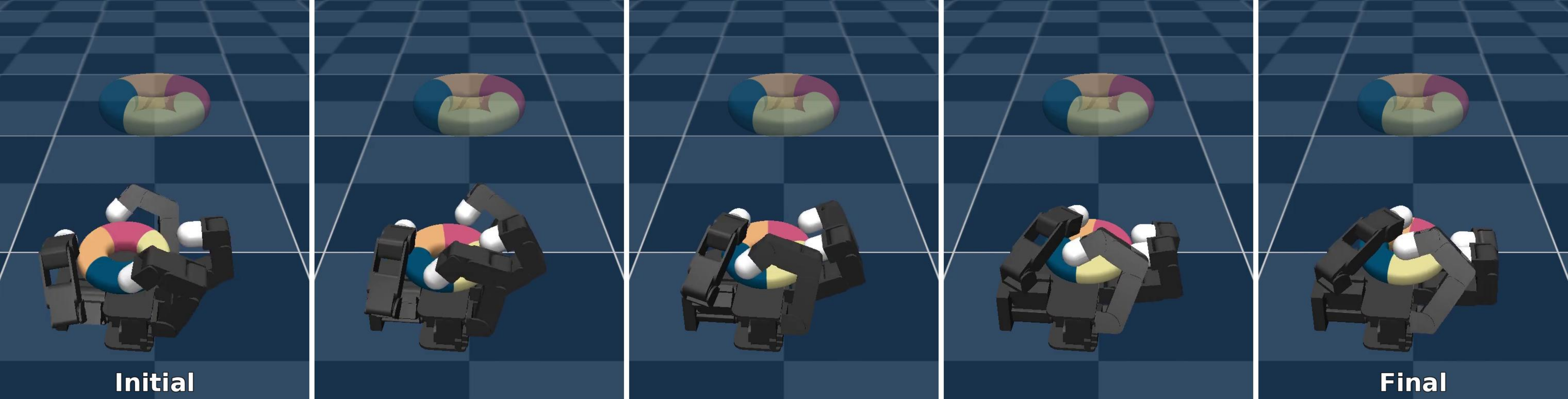}
    \caption{Torus}
    \label{fig:allegro_torus}
  \end{subfigure}

  \caption{Additional Allegro in-hand re-orientation results on: Teapot and Torus.}
  \label{fig:allegro_extra_4}
\end{figure}

\end{document}